\documentclass{amsart}
\usepackage{amsaddr}
\usepackage{amssymb,amsthm}
\usepackage{acronym}
\usepackage{epsfig}
\usepackage[acronym,smallcaps]{glossaries}
\usepackage{lineno,hyperref}
\usepackage{algorithm}
\usepackage{dsfont}
\usepackage{multirow}
\usepackage{arydshln}
\usepackage{subfig}
\usepackage{color}
\usepackage{amsmath}
\usepackage{nomencl}

\newtheorem{theorem}{Theorem}
\newtheorem{corollary}[theorem]{Corollary}

\newtheorem{proposition}[theorem]{Proposition}
\newtheorem{definition}[theorem]{Definition}

\newacronym{AD}{AD}{alternating diffusion}
\newacronym{DM}{DM}{diffusion maps}
\newacronym{RP}{RP}{random projection}

\begin{document}

\title[]{Alternating Diffusion-Maps for multimodal data fusion}
\maketitle




\begin{center}
\normalsize
Ori Katz\textsuperscript{a}, Ronen Talmon\textsuperscript{a},
Yu-Lun Lo\textsuperscript{b} and Hau-Tieng Wu\textsuperscript{c} \par \bigskip

\textsuperscript{a} Viterbi Faculty of Electrical Engineering, Technion -- Israel Institute of Technology, Israel \par
\textsuperscript{b} Department of thoracic medicine, Chang Gung Memorial Hospital, Chang Gung University, School of Medicine, Taipei, Taiwan \par
\textsuperscript{c} Department of Mathematics, University of Toronto, Ontario, Canada \par \bigskip

\end{center}

\begin{abstract}

The problem of information fusion from multiple data-sets acquired by multimodal sensors has drawn significant research attention over the years. In this paper, we focus on a particular problem setting consisting of a physical phenomenon or a system of interest observed by multiple sensors. We assume that all sensors measure some aspects of the system of interest with additional sensor-specific and irrelevant components. Our goal is to recover the variables relevant to the observed system and to filter out the nuisance effects of the sensor-specific variables.
We propose an approach based on manifold learning, which is particularly suitable for problems with multiple modalities, since it aims to capture the intrinsic structure of the data and relies on minimal prior model knowledge.
Specifically, we propose a nonlinear filtering scheme, which extracts the hidden sources of variability captured by two or more sensors, that are independent of the sensor-specific components.
In addition to presenting a theoretical analysis, we demonstrate our technique on real measured data for the purpose of sleep stage assessment based on multiple, multimodal sensor measurements. We show that without prior knowledge on the different modalities and on the measured system, our method gives rise to a data-driven representation that is well correlated with the underlying sleep process and is robust to noise and sensor-specific effects.
\end{abstract}

\section{Introduction}
Often, when measuring a phenomenon of interest that arises from a complex dynamical system, a single data acquisition method is not capable of capturing its entire complexity and characteristics, and it is usually prone to noise and interferences. Recently, due to technological advances, the use of multiple types of measurement instruments and sensors have become more and more popular; nowadays, such equipment is smaller, less expensive, and can be mounted on every-day products and devices more easily.
In contrast to a single sensor, multimodal sensors may capture complementary aspects and features of the measured phenomenon, and may enable us to extract a more reliable and detailed description of the measured phenomenon.

The vast progress in the acquisition of multimodal data calls for the development of analysis and processing tools, which appropriately combine data from the different sensors and handle well the inherent challenges that arise.
One particular challenge is related to the heterogeneity of the data acquired in the different modalities;
datasets acquired from different sensors may comprise different sources of variability, where only few are relevant to the phenomenon of interest.
This particular challenge as well as many others have been the subject of many studies. For a recent comprehensive reviews, see \cite{khaleghi2013multisensor,lahat2015multimodal,gravina2017multi}.

In this paper we consider a setting in which a physical phenomenon is measured by multiple sensors. While all sensors measure the same phenomenon, each sensor consists of different sources of variability;
some are related to the phenomenon of interest, possibly capturing its various aspects, whereas other sources of variability are sensor-specific and irrelevant.
We present an approach based on manifold learning, which is a class of nonlinear data-driven methods, e.g. \cite{Tenenbaum2000,Roweis2000,Donoho2003,Belkin_Niyogi:2003}, and specifically, we use the framework of \gls{DM} \cite{Coifman_Lafon:2006}.
On the one hand, manifold learning is particularly suitable for problems with multiple modalities since it aims to capture the intrinsic
geometric structure of the underlying data and relies on minimal prior model knowledge. This enables to handle multimodal data in a systematic manner, without the need to specially tailor a solution for each modality.
On the other hand, applying manifold learning to data acquired in multiple (multimodal) sensors may capture undesired/nuisance geometric structures as well. Recently, several manifold learning techniques for multimodal data have been proposed \cite{davenport2010joint,keller_audio_visual_2010,yair2017local,salhov2016multi}.
In \cite{davenport2010joint}, the authors suggest to concatenate the samples acquired by different sensors into unified vectors. However this approach is sensitive to the scaling of each dataset, which might be especially diverse among datasets acquired by different modalities.
To alleviate this problem, it is proposed in \cite{keller_audio_visual_2010} to use \gls{DM} to obtain ``standardized'' representation of each dataset separately, and then to concatenate these ``standardized'' representations into the unified vectors.
Despite handling better multimodal data, this concatenation scheme does not utilize the mutual relations and co-dependencies that might exist between the datasets.

While methods such as \cite{davenport2010joint,keller_audio_visual_2010,salhov2016multi} take into account all the measured information,
the methods in \cite{lederman2015alternating,talmon2016latent,yair2017local,lai2000kernel} use local kernels to implement nonlinear filtering.
Specifically, following a recent line of study in which multiple kernels are constructed and combined \cite{de_sa_spectral_2005,de_sa_multi_view_2010,boots_two_manifold_2012,michaeli2015nonparametric}, in \cite{lederman2015alternating,talmon2016latent}, it was shown that a method based on alternating applications of diffusion operators extracts only the common source of variability among the sensors, while filtering out the sensor-specific components.
Therefore we choose to establish our framework based on \gls{DM} which relies on those theoretical foundations. Other nonlinear methods, such as \cite{yair2017local,lai2000kernel}, do not have that theoretical assurance of convergence to an operator that extract the common part, but may also be suitable as a framework. Those methods can be tested and compared empirically with the proposed \gls{DM} based framework in future work.
The shortcoming of alternating applications of diffusion operators arises when having a large number of sensors; often, sensors that measure the same system capture different information and aspects of that system. As a result, the common source of variability among all the sensors captures only a partial or empty look of the system, and important relevant information may be undesirably filtered out.

Here, we address the tradeoff between these two approaches. That is, we aim to maintain the relevant information captured by multiple sensors, while filtering out the nuisance components.
Since the relevance of the various components is unknown, our main assumption is that the sources of variability which are measured only in a single sensor, i.e., sensor-specific, are nuisance. Conversely, we assume that components measured in two or more sensors are of interest.
Importantly, such an approach implements implicitly a smart ``sensor selection''; ``bad'' sensors that are, for example malfunctioned and measure only nuisance information, are automatically filtered out.
These assumptions stem from the fact that the phenomenon of interest is global and not specific to one sensor.
We propose a nonlinear filtering scheme, in which only the sensor-specific sources of variability are filtered out while the sources of
variability captured by two or more sensors are preserved.

Based on prior theoretical results \cite{lederman2015alternating,talmon2016latent}, we show that our scheme indeed accomplishes this task. We illustrate the main features of our method on a toy problem.
In addition, we demonstrate its performance on real measured data in an application for sleep stage assessment based on multiple, multimodal sensor measurements.
Sleep is a global phenomenon with systematic physiological dynamics that represents a recurring non-stationary state of mind and body.
Sleep evolves in time and embodies interactions between different subsystems, not solely limited in the brain.
Thus, in addition to the well-known patterns in electroencephalogram (EEG) signals, its complicated dynamics are manifested in other sensors such as sensors measuring breathing patterns, muscle tones and muscular activity, eyeball movements, etc.
Each one of the sensors is characterized by different structures and affected by numerous nuisance processes as well.
In other words, while we could extract the sleep dynamics by analyzing different sensors, each sensor captures only part of the entire sleep process, whereas it introduces modality artifacts, noise, and interferences.
We show that our scheme allows for an accurate systematic sleep stage identification based on multiple EEG recordings as well as multimodal respiration measurements.
In addition, we demonstrate its capability to perform sensor selection by artificially adding noise sensors.

The remainder of the paper is organized as follows. In Section \ref{sec:Problem_formulation} we present a formulation for the common source extraction problem and present an illustrative toy problem.
In Section \ref{sec:Algorithm}, a brief review for the method proposed in \cite{lederman2015alternating,talmon2016latent} is outlined, and then, a detailed description and interpretation of the proposed scheme are presented.
In Section \ref{sec:Simulation Results}, we first demonstrate the capabilities of the proposed scheme on the toy problem introduced in Section \ref{sec:Problem_formulation}.
Then, in Section \ref{sec:Sleep Results}, we demonstrate the performance in sleep stage identification based on multimodal measured data recorded in a sleep clinic.
Finally, in Section \ref{sec:Conclusions}, we outline several conclusions.

\section{Problem Setting}
\label{sec:Problem_formulation}

Consider a system driven by a set of $K$ hidden random variables $\Theta =\{\boldsymbol{\theta}^{(1)},\boldsymbol{\theta}^{(2)},\ldots,\boldsymbol{\theta}^{(K)}\}$, where $\boldsymbol{\theta}^{(k)} \in \mathbb{R}^{d_k}$.
The system is measured by $M$ observable variables $\boldsymbol{s}^{(m)}, \ m=1,\dots,M$,
where each sensor has access to only a partial view of the entire system and its driving variables $\Theta$.
To formulate it, we define a ``sensitivity table'' given by the binary matrix $\mathbf{S} \in \mathbb{Z}^{K\times M}_2$, indicating the variables sensed by each observable variable.
Specifically, the $(k,m)$th element in $\mathbf{S}$ indicates whether the hidden variable $\boldsymbol{\theta}^{(k)}$ is measured by the observable variable $\boldsymbol{s}^{(m)}$. It should be noted that this binary notation is a rough simplification since that there are soft degrees of observability. However, at least theoretically, when we have sufficeint data the algorithm is guaranteed to work for any bilishitz observation function. When the data amount is limited, those degrees of observability are dominated by the Lipschitz constants and the signal-to-noise ratio. Some of these observalities issues were treated in \cite{dsilva2015data} .Further quantification of those parameters for the derivation of soft observability scores is beyond the scope of this article and may be addressed in a future work.
The observable variables are therefore given by
\begin{equation}\label{eq:observer_formula}
  	\boldsymbol{s}^{(m)}=h_{m}(\Theta^{(m)},\boldsymbol{n}^{(m)}) \in \mathbb{R}^{D_m}
\end{equation}
where $h_m(\cdot)$ is a bilipschitz observation function, $\boldsymbol{n}^{(m)} \in \mathbb{R}^{p_m}$ are hidden random variables captured only by the $m$th observable variable, and $\Theta^{(m)}$ is the subset of driving hidden variables of interest sensed by $\boldsymbol{s}^{(m)}$, given by
\begin{equation}\label{eq:Theta_subset_def}
	\Theta^{(m)} = \left\{ \boldsymbol{\theta}^{(k)} | \forall k, S_{k,m}=1 \right\} \subseteq \Theta , m=1,\ldots,M
\end{equation}
The random hidden variables $\boldsymbol{n}^{(m)}$  are \textit{sensor-specific} (associated only with the $m$th observer). They are conditionally independent given the hidden variables of interest and will be assumed as noise/nuisance variables.
We further assume that each random hidden variable in $\Theta$ is measured by at least two observable variables, such that $\sum_{m=1}^{M}{S_{k,m}}\geq 2$ for each $k=1,\dots,K$.
As a result, we refer to the hidden variables $\boldsymbol{\theta}^{(k)}$ in $\Theta$ as {\em common variables}.

In order to simplify the notation, we denote the subset of all hidden variables (both common and sensor-specific) measured by the $m$th observable by $\mathcal{S}^{(m)}=\{\Theta^{(m)},\boldsymbol{n}^{(m)}\}$.
Furthermore, we assume that the dimensions of the observations and the hidden variables satisfy
\begin{equation}\label{eq:dim_reqirement}
	D_m \geq \sum_{\boldsymbol{\theta}^{(k)} \in \Theta^{(m)}}{(d_k+p_k)}, \  k=1,2,...,M
\end{equation}
i.e., the observations are in higher dimension than the hidden common and nuisance variables.

\begin{table}[t]
	\caption{List of important notation.}
\begin{tabular}{ |p{3cm} p{10cm}| }
	\hline
	\multicolumn{2}{|l|}{Nomenclature} \\
	\hline
	\hline
	$K$ & number of common hidden variables  \\
	$\boldsymbol{\theta}^{(k)}$ & $k$th common hidden variable \\
	$d_k$  & dimension of $\boldsymbol{\theta}^{(k)}$ \\
	$\Theta$ & set of all common hidden variables\\ 
	$M$ & number of observable variables  \\
	$\boldsymbol{s}^{(m)}$ & $m$th observable variable \\
	$D_m$  & dimension of $\boldsymbol{s}^{(m)}$  \\
	$\mathbf{S}$  &  sensitivity table \\
	$h_m(\cdot)$ & a bilipschitz $m$th observation function   \\
	$\boldsymbol{n}^{(m)}$  &  $m$th sensor-specific hidden (nuisance) variables \\
	$p_m$  &   dimension of $\boldsymbol{n}^{(m)}$\\
	$\Theta^{(m)}$ & subset of $\Theta$ sensed by $\boldsymbol{s}^{(m)}$ \\
	$\mathcal{S}^{(m)}$  & subset of all hidden variables measured by $\boldsymbol{s}^{(m)}$  \\
	\hline
\end{tabular}
\end{table}

An observation of the system denoted as $(\boldsymbol{s}^{(1)}_i,\boldsymbol{s}^{(2)}_i,\ldots,\boldsymbol{s}^{(M)}_i)$ is associated with a realization of the hidden variables $\Theta_i=(\boldsymbol{\theta}^{(1)}_i, \ldots, \boldsymbol{\theta}^{(K)}_i)$  and realizations of the $M$ hidden nuisance variables $(\boldsymbol{n}^{(1)}_i, \ldots, \boldsymbol{n}^{(M)}_i)$.
Given $N$ observation samples $\left\{(\boldsymbol{s}^{(1)}_i,\boldsymbol{s}^{(2)}_i,\ldots,\boldsymbol{s}^{(M)}_i)\right\}_{i=1}^N$, our goal is to obtain a parametrization for the underlying realizations of the common hidden random variables $\left\{ ( \boldsymbol{\theta}^{(1)}_i, \ldots, \boldsymbol{\theta}^{(K)}_i) \right\}_{i=1}^N$ while filtering out the nuisance variables $\left\{ ( \boldsymbol{n}^{(1)}_i, \ldots, \boldsymbol{n}^{(M)}_i ) \right\}_{i=1}^N$.
We note that the observations index $i$ may represent the time index in case of time series.

\subsection{Illustrative toy problem}
\label{ssec:toy_problem}

We illustrate the problem setting using the following toy example. Consider six rotating arrows captured in simultaneous snapshots by three different cameras. We assume that each arrow rotates at different speed, and that each camera can capture only a partial image of the entire system. The partial view of each camera is depicted in Figure~\ref{fig:toy_settings}.
Thus, overall, each camera captures a sequence of snapshots (a movie) of three rotating colored arrows. Further illustration of the entire system and of the captured images by each camera can be seen in the following link \href{https://youtu.be/a-yb7ScdnnA}{https://youtu.be/a-yb7ScdnnA}.

\begin{figure}[t]
\centering
\begin{tabular}{cc}
\hspace{-0.1in} \includegraphics[scale=0.3]{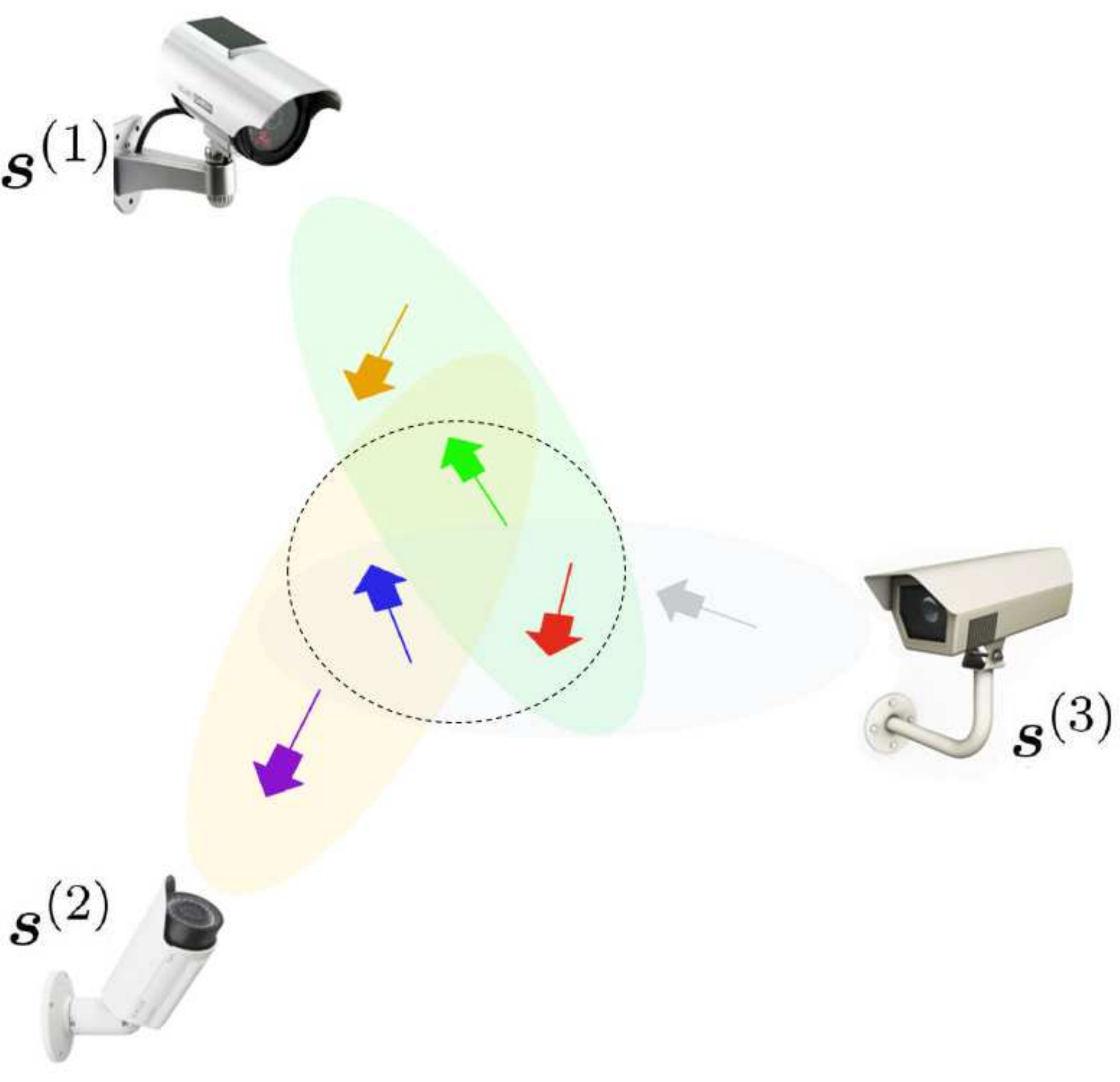} &
\vspace{-.1in} \includegraphics[scale=0.3]{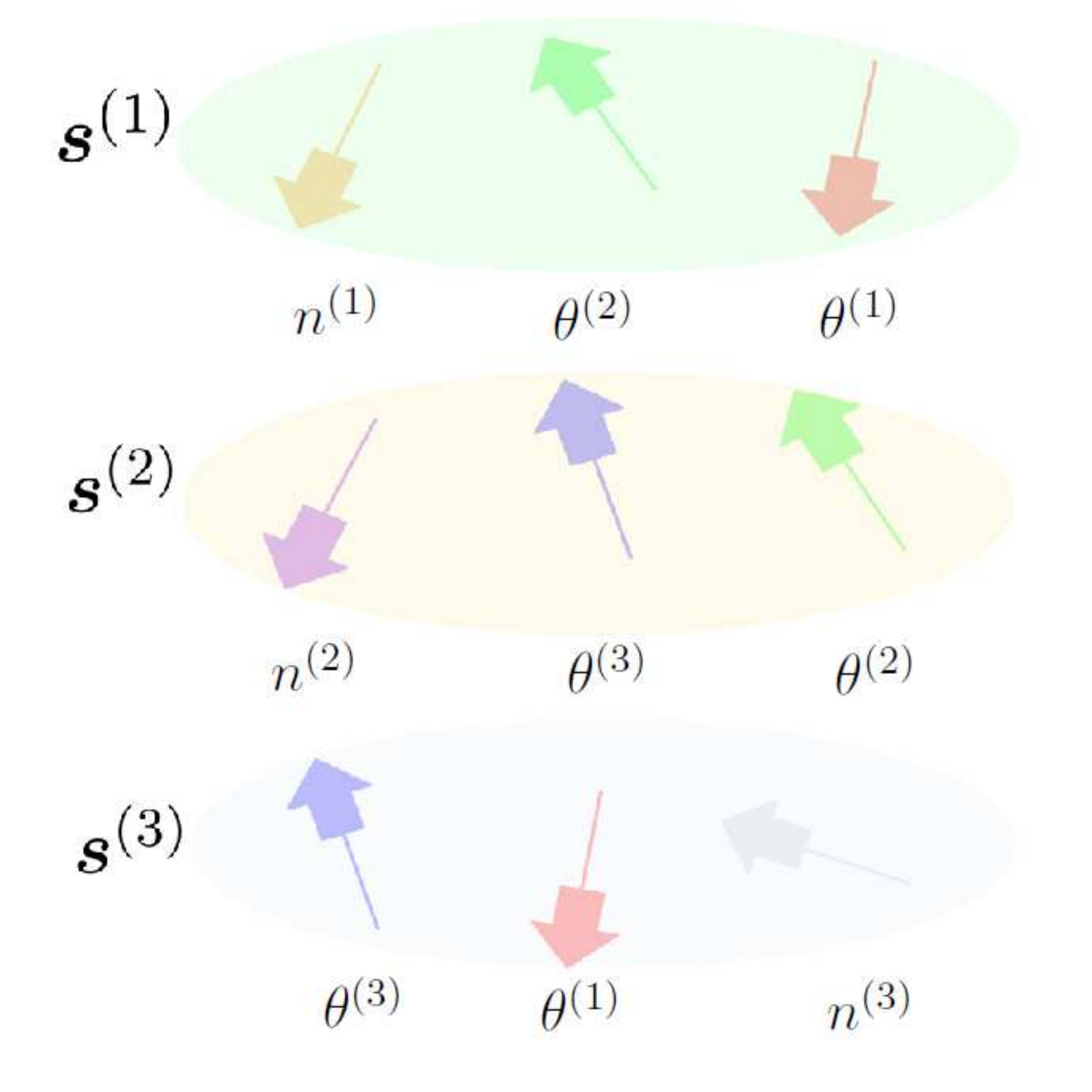}\\
(a) & (b) \\
\end{tabular}
\caption {Toy problem setup. (a) The coverage area of each camera, the system's range of interest is marked by the dashed circle. (b) Sample snapshot taken by each camera.}
\label{fig:toy_settings}
\end{figure}
In this problem setting, the hidden variables are the six rotation angles of the arrows: the common variables $\Theta=\{\theta^{(1)},\theta^{(2)},\theta^{(3)}\}$ are the three rotation angles of the centred arrows, which are marked by the dashed circle in Figure~\ref{fig:toy_settings}, and the nuisance variables $\{n^{(1)},n^{(2)},n^{(3)}\}$ are the three rotation angles of the peripheral arrows, since each is captured only by a single camera. It should be noted that none of the arrows is common to all of the cameras, meaning that the set of common components within the entire set of observables is empty.

In order to identify the hidden variables, we use different colors for the arrows. The arrows rotating according to the common variables $\Theta=\{\theta^{(1)},\theta^{(2)},\theta^{(3)}\}$ are colored in red, green and blue, respectively, and the arrows rotating according to the nuisance variables $\{n^{(1)},n^{(2)},n^{(3)}\}$ are colored in orange, purple and gray, respectively. The hidden variables measured by each camera are $\mathcal{S}^{(1)}=\{\theta^{(1)},\theta^{(2)},n^{(1)}\}$, $\mathcal{S}^{(2)}=\{\theta^{(2)},\theta^{(3)},n^{(2)}\}$ and $\mathcal{S}^{(3)}=\{\theta^{(3)},\theta^{(1)},n^{(3)}\}$. Our goal is to obtain a parametrization of the rotation angles of the three common arrows $\Theta=\{\theta^{(1)},\theta^{(2)},\theta^{(3)}\}$ given the three movies of the cameras, without any prior knowledge on the system and the problem structure. In the sequel, we will use this toy problem for demonstrating important aspects and how our method accomplishes this task.

\section{Nonlinear Filtering Scheme}
\label{sec:Algorithm}
\subsection{Diffusion Maps}
\label{subsec:dm}

\Gls{DM} is a non-linear data-driven dimensionality reduction method \cite{Coifman_Lafon:2006}.
Assume we have $N$ high-dimensional data-points $\left\{\boldsymbol{s}_i\right\}_{i=1}^{N}$.
The \gls{DM} method begins with the calculation of a pairwise affinity matrix based on a local kernel, often using some metric within a gaussian kernel, i.e.,
\begin{equation}\label{eq:DM_affinity_def}
  W_{i,j}=\exp\left(-\frac{\|\boldsymbol{s}_i-\boldsymbol{s}_j\|_{M}^2}{\varepsilon}\right),
\end{equation}
where $\varepsilon >0$ is a tuneable kernel scale and  $\|\cdot\|_{M}$ is a metric. The choice of the metric $\|\cdot\|_{M}$ depends on the application; common choices are the Euclidean and the Mahalanobis distances \cite{Coifman_Lafon:2006,Coifman_Singer:2008,TalmonPNAS,TalmonACHA,talmon2015manifold}. This construction implicitly defines a weighted graph, where the data samples $\left\{\boldsymbol{s}_i\right\}_{i=1}^{N}$ are the nodes of the graph, and $W_{i,j}$ is the weight of the edge connecting node $\boldsymbol{s}_i$ and node $\boldsymbol{s}_j$. The next step is to normalize the affinity matrix and then to build the diffusion operator $\mathbf{K} \in \mathbb{R}^{N\times N}$, e.g., by:
\begin{equation}
\label{eq:DM_affinity_def2}
Q_{i,i}=\left(\sum\limits_{l=1}^{N} W_{i,l}\right)^{-1} ; \mathbf{K} = \mathbf{Q}\mathbf{W},
\end{equation}
where $\mathbf{Q}$ is a diagonal matrix used for normalization, such that in this case $\mathbf{K}$ is row-stochastic. Hence, $\mathbf{K}$ can be viewed as the transition matrix of a Markov chain defined on the graph. Accordingly, for $t > 0$, $\mathbf{K}^t$ is the transition probability matrix of $t$ consecutive steps, and $(K^t)_{i,j}$ is the  probability to jump from node $\boldsymbol{s}_i$ to node $\boldsymbol{s}_j$ in $t$ steps.
Let $d_t(i,j)$ be the diffusion distance \cite{Coifman_Lafon:2006} between the $i$th and the $j$th data samples, i.e. $d_t$ is a function defined by
\begin{equation}
\label{eq:DM_DD}
	d_t(i,j)=\sqrt{\sum_{l=1,\ldots,N} {\frac{\left((K^t)_{i,l}-(K^t)_{j,l}\right)^2}{\phi_0(l)}}}
\end{equation}
where $\phi_0(\cdot)$ is the stationary distribution of the Markov chain.
The diffusion distance has been shown to be a powerful metric for measuring geometrical similarities between data-points \cite{Coifman_Lafon:2006}. While the Euclidean distance compares two individual data-points and might be affected by distortions and noise, the diffusion distance introduces much more noise-robust affinities since it relies on the connectivity between the two data-points using the entire data-set \cite{Coifman_Lafon:2006,Ht_ElKaroui:2016}.

However, the direct computation of the diffusion distance is cumbersome.
An efficient calculation is attainable via the spectral decomposition of $\mathbf{K}$.
Let $\left\{\lambda_l\right\}_{l=0}^{N-1}$ and $\left\{\boldsymbol{\psi}_l\right\}_{l=0}^{N-1}$ be the sets of eigenvalues and right eigenvectors of $\mathbf{K}$, where the eigenvalues are in descending order.
Define a new representation (embedding) of the data-points:
\begin{equation}\label{}
  \Psi_t(i) : \boldsymbol{s}_i \mapsto [\lambda_{1}^{t}\psi_1(i),\lambda_{2}^{t}\psi_2(i),\ldots,\lambda_{N-1}^{t}\psi_{N-1}(i)],
\end{equation}
where $\psi_l(i)$ denotes the $i$the element of $\boldsymbol{\psi}_l$.
The obtained embedding provides a new representation of the data, referred to as \gls{DM}, in which the Euclidean distance between two embedded data-point is equal to the diffusion distance \cite{Coifman_Lafon:2006}, i.e.:
\begin{equation}
d_t^2(i,j) = \|\Psi_t(i)-\Psi_t(j)\|^2 = \sum_{l\geq1} \lambda_{l}^{2t}(\psi_l(i)-\psi_l(j))^2.
\end{equation}

In order to achieve a compact representation in reduced dimensionality, \gls{DM} is often redefined by keeping only the first $L$ components (i.e., the $L$ eigenvalues and eigenvectors corresponding to the largest $L$ eigenvalues):
\begin{equation}\label{eq:dm}
  \Psi_t(i) : \boldsymbol{s}_i \mapsto [\lambda_{1}^{t}\psi_1(i),\lambda_{2}^{t}\psi_2(i),\ldots,\lambda_{L}^{t}\psi_{L}(i)],
\end{equation}
where $L$ is usually determined by the eigenvalues decay.
For more details and full analysis of this algorithm see \cite{Coifman_Lafon:2006,Ht_Singer:2016}. The entire \gls{DM} method is outlined in Algorithm \ref{alg:DM}.

The term ``diffusion distance" in \eqref{eq:DM_DD} suggests that $d_t(i,j)$ induces a reasonable notion of distance.  Recall the definition of a distance.
\begin{definition}\label{def:metric}
Let $X$ be a set.  A {\it distance} (or {\it metric}) on $X$ is a function $d : X \times X \to \mathbb{R}_+$ such that for all $x, y, z \in X$:
\begin{enumerate}
\item{
$d(x, y) = 0$ if and only if $x = y$,
}
\item{
$d(x, y) = d(y, x)$,
}
\item{
$d(x, z) \leq d(x, y) + d(y, z)$.
}
\end{enumerate}
\end{definition}

The following proposition states that the ``diffusion distance" is really a metric defined on the nodes of the graph.

\begin{proposition}
\label{pre:DD_Metric}
If  $\mathbf{K}$ is full rank, then $d_t$ is a distance function.
\end{proposition}
Since we could not find a proof in the literature, for the sake of self-containment, we provide a proof that summarizes the discussion in \cite{LectureNotesAmit}.
\begin{proof}
We prove that (1)-(3) hold. Define $\widetilde{\mathbf{K}}^t = \mathbf{\Phi}^{-1}\mathbf{K}^t$ where $\mathbf{\Phi}$ is a diagonal matrix such that $\Phi_{k,k}=\sqrt{\phi_0(k)}$. Since $\mathbf{K}^t$ is full rank and since $\mathbf{\Phi}$ is non-degenerate by the construction of the weighted graph, $\widetilde{\mathbf{K}}^t$ is full rank. Denote the $i$th row of $\widetilde{\mathbf{K}}^t$ as $\boldsymbol{v}_i$. Accordingly, $d_t(i,j)$ can be expressed as the Euclidean distance between the $i$th and the $j$th rows of $\widetilde{\mathbf{K}}^t$:
\begin{equation}\label{eq:dd_proof}
	d_t(i,j)=\sqrt{\sum_{l=1,\ldots,N} {\left((\widetilde{\mathbf{K}}^t)_{i,l}-(\widetilde{\mathbf{K}}^t)_{j,l}\right)^2}} = \|\boldsymbol{v}_i-\boldsymbol{v}_j\|_{\mathbb{R}^N}
\end{equation}
The properties of the Euclidean distance in \eqref{eq:dd_proof} imply that (2) and (3) hold.
If $i = j$, then $d_t(i, j) = 0$. 
If $d_t(i,j) = 0$, then $\|\boldsymbol{v}_i-\boldsymbol{v}_j\|^2=0$, implying that $\boldsymbol{v}_i=\boldsymbol{v}_j$.
Since $\widetilde{\mathbf{K}}^t$ is full rank, there are no identical columns. In other words, no two different samples $\boldsymbol{v}_i$ and $\boldsymbol{v}_j$ for $i \neq j$ have identical affinities to all other samples, i.e., $\boldsymbol{v}_i \neq \boldsymbol{v}_j$.
Therefore, if  $\boldsymbol{v}_i=\boldsymbol{v}_j$, then $i=j$. 
\end{proof}

\subsection{Alternating Diffusion}
Consider a system similar to the one described in Section \ref{sec:Problem_formulation}, with only $M=2$ observable variables and $K=1$ common variable.
The \gls{AD} algorithm, outlined in Algorithm \ref{alg:AD}, builds from the observations an \gls{AD} operator that is equivalent to a simple diffusion operator (as described in Section \ref{subsec:dm}) that would have been computed if we had a direct access to samples of the common hidden variables.
This operator enables to capture only the structure of the common variables while ignoring the nuisance (sensor-specific) variables.
For more details and full analysis of this algorithm see \cite{lederman2015alternating,talmon2016latent};
here, we only bring a brief review of the method and the construction of the \gls{AD} operator.

Assume we have $N$ aligned samples (realizations) from $2$ observable variables: $\left\{ (\boldsymbol{s}^{(1)}_i,\boldsymbol{s}^{(2)}_i)\right\}_{i=1}^N$. For each observation we build an affinity matrix $\mathbf{W}^{(1)}$ and $\mathbf{W}^{(2)}$ as follows:\begin{equation}
  W_{i,j}^{(1)}=\exp\left(-\frac{\|\boldsymbol{s}^{(1)}_i-\boldsymbol{s}^{(1)}_j\|_{M_1}^2}{\varepsilon^{(1)}}\right) ; W_{i,j}^{(2)}=\exp\left(-\frac{\|\boldsymbol{s}^{(2)}_i-\boldsymbol{s}^{(2)}_j\|_{M_2}^2}{\varepsilon^{(2)}}\right)
\end{equation}
for all $i,j=1,\hdots,N$, where $\varepsilon^{(1)}$ and $\varepsilon^{(2)}$ are the tuneable kernel scales and  $\|\cdot\|_{M_1}$ and $\|\cdot\|_{M_2}$ are the chosen metrics for each set of observations.
Based on the affinity matrix, we calculate the diffusion operators $\mathbf{K}^{(1)}$ and $\mathbf{K}^{(2)}$ according to:
\[
\begin{tabular}{ccc}
    $Q_{i,i}^{(1)}=\left(\sum\limits_{l=1}^{N} W_{i,l}^{(1)}\right)^{-1}$&;& $Q_{i,i}^{(2)}=\left(\sum\limits_{l=1}^{N} W_{i,l}^{(2)}\right)^{-1}$ \\
    $\mathbf{K}^{(1)}= \mathbf{Q}^{(1)}\mathbf{W}^{(1)} $&;& $\mathbf{K}^{(2)}= \mathbf{Q}^{(2)}\mathbf{W}^{(2)} $
\end{tabular}
\]
where $\mathbf{Q}^{(1)}$ and $\mathbf{Q}^{(2)}$ are diagonal matrices used for normalization. Next, we define $\mathbf{K}^{(1)\bigcap(2)}=\mathbf{K}^{(1)}\mathbf{K}^{(2)}$ as the \gls{AD} operator. Note that $\mathbf{K}^{(1)\bigcap(2)}$ is row-stochastic, and hence, can be considered as a transition probability matrix of a new Markov chain that alternates between the two data sets. Namely, each step of this alternating process consists of a propagation step using $\mathbf{K}^{(1)}$ followed by a propagation step using $\mathbf{K}^{(2)}$.

Broadly, in each propagation step, the Markov chain jumps with high probability to neighboring samples that are similar in terms of the kernel. Combining alternating steps results in consecutive jumps according to similarities in the first set and then according to similarities in the second set. Overall, only similarities in terms of the common components among the two views are maintained.

Formally, we define the diffusion distance between the $i$th and the $j$th sample based on the \gls{AD} operator as the following Euclidean distance
\begin{equation}\label{eq:ad_diff_dist}
  d^{(1)\bigcap(2)}_{t}(i,j) = \sqrt{\sum_{l=1,\ldots,N} {\frac{\Big(\big((K^{(1)\bigcap(2)})^t\big)_{i,l}-\big((K^{(1)\bigcap(2)})^t\big)_{j,l}\Big)^2}{\phi^{(1)\bigcap(2)}_0(l)}}}
\end{equation}
where $\boldsymbol{\phi}^{(1)\bigcap(2)}_0$ is the stationary distribution of $\mathbf{K}^{(1)\bigcap(2)}$ and $t>0$ is the number of alternating steps. The following corollary is an immediate results of Proposition \ref{pre:DD_Metric}.
\begin{corollary}
\label{pre:ADD_Metric}
If $\mathbf{K}^{(1)\bigcap(2)}$ is full rank, then $d^{(1)\bigcap(2)}_{t}$ is a distance function.
\end{corollary}
It can be shown that this distance is equivalent to the diffusion distance that would have been computed if we had a direct access to observable variables that see only the common variable \cite{lederman2015alternating}.

\begin{algorithm}[t]
\caption{Diffusion Maps}
\label{alg:DM}
\textbf{Input}: High-dimensional samples from an observable variables: $\left\{\boldsymbol{s}_i\right\}_{i=1}^N$.\\
\textbf{Output}: $L$ dimensional representation of the data-set $\left\{\Psi_t(i)\right\}_{i=1}^N$ where $\Psi_t(i) \in \mathbb{R}^{L}$.

\begin{enumerate}
\item Calculate the affinity matrix $\mathbf{W}$:
\begin{equation}
\label{eq:affinity_DM}
 W_{i,j}=\exp\left(-\frac{\|\boldsymbol{s}_i-\boldsymbol{s}_j\|_{M}^2}{\varepsilon}\right)
 \end{equation}
\item Compute the diffusion operator (transition matrix) $\mathbf{K}$:
\begin{equation}
Q_{i,i}=\left(\sum\limits_{l=1}^{N} W_{i,l}\right)^{-1} ; \mathbf{K} = \mathbf{Q}\mathbf{W},
\end{equation}
\item Calculate the spectral decomposition of $\mathbf{K}$ and obtain its eigenvalues $\left\{\lambda_l\right\}_{l=0}^{N-1}$ and eigenvectors $\left\{\psi_l\right\}_{l=0}^{N-1}$.
\item Define a new embedding for the data-points:
\begin{equation}\label{}
  \Psi_t(i) : \boldsymbol{s}_i \mapsto [\lambda_{1}^{t}\psi_1(i),\lambda_{2}^{t}\psi_2(i),\ldots,\lambda_{L}^{t}\psi_{L}(i)]
\end{equation}
where $t>0$ is a selected number of steps and $\psi_l(i)$ denotes the $i$th element of $\psi_l$.
\end{enumerate}
\end{algorithm}

\begin{algorithm}[th]
\caption{Alternating Diffusion}
\label{alg:AD}
\textbf{Input}: Aligned samples from $2$ observable variables: $\left\{ (\boldsymbol{s}^{(1)}_i,\boldsymbol{s}^{(2)}_i )\right\}_{i=1}^N$.\\
\textbf{Output}: Diffusion distances $d^{(1)\bigcap(2)}_{t}$.

\begin{enumerate}

\item Calculate two pairwise affinity matrices $\mathbf{W}^{(1)}$ and $\mathbf{W}^{(2)}$ based on a gaussian kernel as follows:
\begin{equation}\label{eq:affinity_def}
  W_{i,j}^{(1)}=\exp\left(-\frac{\|\boldsymbol{s}^{(1)}_i-\boldsymbol{s}^{(1)}_j\|_{M}^2}{\varepsilon^{(1)}}\right) ; W_{i,j}^{(2)}=\exp\left(-\frac{\|\boldsymbol{s}^{(2)}_i-\boldsymbol{s}^{(2)}_j\|_{M}^2}{\varepsilon^{(2)}}\right)
\end{equation}
for all $i,j=1,\hdots,N$, where $\varepsilon^{(1)}$ and $\varepsilon^{(2)}$ are the kernel scales and  $\|\cdot\|_{M}^2$ is the chosen metric.

\item Create two diffusion operators $\mathbf{K}^{(1)}$ and $\mathbf{K}^{(2)}$:
\[
\begin{tabular}{ccc}\label{eq:kernel_def}
    $Q_{i,i}^{(1)}=\left(\sum\limits_{l=1}^{N} W_{i,l}^{(1)}\right)^{-1}$&;& $Q_{i,i}^{(2)}=\left(\sum\limits_{l=1}^{N} W_{i,l}^{(2)}\right)^{-1}$ \\
    $\mathbf{K}^{(1)}= \mathbf{Q}^{(1)}\mathbf{W}^{(1)} $&;& $\mathbf{K}^{(2)}= \mathbf{Q}^{(2)}\mathbf{W}^{(2)}$
\end{tabular}
\]

\item Build the alternating-diffusion kernel:
\begin{equation}\label{eq:commom_ker_def}
  \mathbf{K}^{(1)\bigcap(2)}=\mathbf{K}^{(1)}\mathbf{K}^{(2)}
\end{equation}

\item Compute the altenating-diffusion distance between each two points $(i,j)$
\begin{equation}\label{eq:common_dist}
  d^{(1)\bigcap(2)}_{t}(i,j) = \sqrt{\sum_{l=1,\ldots,N} {\frac{\Big(\big((K^{(1)\bigcap(2)})^t\big)_{i,l}-\big((K^{(1)\bigcap(2)})^t\big)_{j,l}\Big)^2}{\phi^{(1)\bigcap(2)}_0(l)}}}
\end{equation}
where $\boldsymbol{\phi}^{(1)\bigcap(2)}_0$ is the stationary distribution of $\mathbf{K}^{(1)\bigcap(2)}$ and $t >0 $ is a tuneable parameter.
\end{enumerate}
\end{algorithm}

\subsection{Common Graph}
\label{subsec:common_graph}

\gls{AD} provides us with an access to the common variables between a pair of observable variables. By using \gls{AD} as a building block, we propose a generalization for a set of multiple observable variables.
Consider the system described in Section \ref{sec:Problem_formulation} with aligned samples from $M$ observable variables: $\left\{ (\boldsymbol{s}^{(1)}_i,\boldsymbol{s}^{(2)}_i,\ldots,\boldsymbol{s}^{(M)}_i ) \right\}_{i=1}^N$. The observable variables are driven by a set of $K$ hidden random variables $\Theta = (\boldsymbol{\theta}^{(1)},\boldsymbol{\theta}^{(2)},\ldots,\boldsymbol{\theta}^{(K)})$ and contaminated by a set of $M$ nuisance sensor-specific variables $( \boldsymbol{n}^{(1)}, \ldots, \boldsymbol{n}^{(M)} )$.
Our goal is to obtain a parametrization of the common hidden random variables $\Theta$ from the observations.

Corollary \ref{pre:ADD_Metric} provides the analytic foundation and justification to the method presented in this paper.
More specifically, in the context of our problem, consider a pair of observable variables $\boldsymbol{s}^{(m)}$ and $\boldsymbol{s}^{(n)}$.
Applying \gls{AD} to $\boldsymbol{s}^{(m)}$ and $\boldsymbol{s}^{(n)}$ yields the common hidden variables measured by the two. Therefore, its operation can be written as
\begin{equation} \label{eq:AD_intersection}
    \mathcal{S}^{(m)}\bigcap\mathcal{S}^{(n)}=\Theta^{(m)}\bigcap\Theta^{(n)}
\end{equation}
In other words, \gls{AD} captures only a subset of the common hidden variables $\Theta^{(m)}\bigcap\Theta^{(n)}$, and in addition, filters out the nuisance variables $\boldsymbol{n}^{(m)}$ and $\boldsymbol{n}^{(n)}$, which are specific to each observation.

The main idea in our method is based on the fact that the desired set of variables $\Theta$ can be derived from the union of the pairwise intersections between all pairs, meaning that:
\begin{equation}\label{eq:Theta}
     \Theta=\bigcup_{m \neq n}\left(\Theta^{(m)}\bigcap\Theta^{(n)}\right).
\end{equation}
A direct implementation of the scheme in \eqref{eq:Theta} is not feasible, since the pairwise intersections of $\Theta^{(m)}$ and $\Theta^{(n)}$ are not accessible to us.
However, note that by substituting \eqref{eq:AD_intersection} in \eqref{eq:Theta}, we get
\begin{equation}\label{eq:ThetaObs}
\Theta=\bigcup_{m \neq n}\left(\mathcal{S}^{(m)}\bigcap\mathcal{S}^{(n)}\right)
\end{equation}
meaning that $\Theta$ can be expressed using the accessible observations sets $\mathcal{S}^{(m)}$ through the
union of the intersections of all possible pairs.
Thus, this scheme for recovering $\Theta$ can be implemented by multiple applications of \gls{AD} to all possible pairs of observable variables.

The union is implemented through the formulation of a new kernel in which the affinity between each pair of samples is given by the sum of the diffusion distances over all pairs of observations.
Therefore for each kernel resulting from an application of \gls{AD} to a single pair of observations, we compute the following diffusion distance $d^{(m)\bigcap(n)}_t$, similarly to \eqref{eq:ad_diff_dist}
\begin{equation}\label{eq:ad_diff_dist2}
     d^{(m)\bigcap(n)}_t(i,j) =\sqrt{\sum_{l=1}^{N} \frac{\Big( \big((K^{(m)\bigcap(n)})^t\big)_{i,l}-\big((K^{(m)\bigcap(n)})^t\big)_{j,l} \Big) ^2}{\phi^{(m)\bigcap(n)}_0(l)}},
\end{equation}
where $\boldsymbol{\phi}^{(m)\bigcap(n)}_0$ is the stationary distribution of $\mathbf{K}^{(m)\bigcap(n)}$ and $t > 0$ is a tuneable parameter indicating the number of \gls{AD} steps.
We then define the \textit{common diffusion distance} $d^{(\cup)}_t$ as a summation over the alternating diffusion distances \eqref{eq:ad_diff_dist2} resulting from applications to all possible pairs of observations, according to
\begin{equation}
	\label{eq:ProposedUnionMetric}
    d^{(\cup)}_t(i,j) =\sum_{1\leq m,n\leq M, m \neq n} d^{(m)\bigcap(n)}_t(i,j)
\end{equation}
where $i,j=1,\ldots,N$.
We now show that $d^{(\cup)}_t$ is a metric.
\begin{proposition}
\label{prop:metrics_sum}
Let $X$ be a set and consider two distance functions $d_1,d_2:  X \times X \to \mathbb{R}_+$.
Define $d(x,y)=d_1(x,y)+d_2(x,y)$ for all $x, y \in X$. Then $d$ is a distance function as well.
In particular, if $\mathbf{K}^{(m)\bigcap(n)}$ are full rank for all $m,n=1,\ldots,M, m \neq n$, then $d^{(\cup)}_t$ is a distance function.
\end{proposition}
\begin{proof}
By definition $d$ is $d : X \times X \to \mathbb{R}_+$.
We prove that properties (1)--(3) in Definition \ref{def:metric} hold.
Using the symmetry property of $d_1$ and $d_2$ we have that $d(x,y)=d(y,x)$.
Consider $x,z \in X$, using property (3) of $d_1$ and $d_2$, for any $y \in X$
 $d(x,z)=d_1(x,z)+d_2(x,z) \leq d_1(x,y)+d_1(y,z)+d_2(x,y)+d_2(y,z) = d(x,y) + d(y,z)$.
If $x=y$ then $d(x,y)=0$. If $d(x,y)=0$, using the non-negativity property (1) of $d_1$ and $d_2$ we have that $d_1(x,y)=0$ and $d_2(x,y)=0$. From property (1) we obtain that $x=y$.

Now, by Corollary \ref{pre:ADD_Metric}, if $\mathbf{K}^{(m)\bigcap(n)}$ is full rank, then, $d^{(m)\bigcap(n)}_t$ is a distance function, and therefore, by a straight-forward generalization, it follows that $d^{(\cup)}_t$ is a distance function.
\end{proof}

Based on the common diffusion distance $d^{(\cup)}_t$, then we calculate an affinity matrix
\begin{equation}
  W^{(\cup)}_{i,j}=\exp\left(-\frac{d^{(\cup)}_t(i,j)}{\varepsilon^{(\cup)}}\right),
\end{equation}
where $\varepsilon^{(\cup)} >0$ is the chosen kernel scale. Next we normalize the affinity matrix and build the common diffusion operator $\mathbf{K^{(\cup)}} \in \mathbb{R}^{N\times N}$
\begin{equation}
    Q^{(\cup)}_{i,i}=\left(\sum\limits_{l=1}^{N} W^{(\cup)}_{i,l}\right)^{-1}  ;  \mathbf{K^{(\cup)}}= \mathbf{Q}^{(\cup)}\mathbf{W}^{(\cup)}
\end{equation}
In conclusion, the new graph with kernel $\mathbf{K^{(\cup)}}$ consists of two main components. First, the intersections between any pair of observations $\mathcal{S}^{(m)}\bigcap\mathcal{S}^{(n)}$ are implemented using \gls{AD} that provides the extraction of the common hidden variables $\Theta^{(m)}\bigcap\Theta^{(n)}$. Second, the union $\bigcup_{m,n}\left(\Theta^{(m)}\bigcap\Theta^{(n)}\right)$ is implemented via the summation of the resulting diffusion distances from the \gls{AD} applications.
By construction, in the kernel $\mathbf{K^{(\cup)}}$, the connectivity between the $i$th and the $j$th data samples is proportional to the intrinsic distance $\|\Theta_i-\Theta_j\|$. This means that the common global diffusion kernel $\mathbf{K^{(\cup)}}$ can be used for obtaining a low-dimensional representation of $\Theta$.
The common graph algorithm described in this section is summarized in Algorithm \ref{alg:CG}.

Four final remarks follow. First, it should be noted that there is a theoretical gap between the desired union described in \eqref{eq:ThetaObs} and its implementation via the summation of the common diffusion distances in the proposed metric $d^{(\cup)}_t$ which is described in \eqref{eq:ProposedUnionMetric}. The motivation for this choice for implementation is that the embeddings achieved using the proposed metric $d^{(\cup)}_t$ corresponds to those that would have been achieved using a union scheme. Although that this claim is supported by empirical results in Section \ref{sec:Simulation Results}, the derivation of a union scheme still calls for rigorous analysis in future work.

Second, the proposed implementation of the union via diffusion distance summation enhances the common variables that appear multiple times in the various intersections.
By doing so, we slightly abuse the definition of the union, where duplicates are all ``put together''. In other words, in the strict definition of a union, in contrast to our implementation, common hidden variables related to two or more intersection results should be taken into account only once.
Depending on the application at hand, this may be a desired property, and the derivation of a scheme in which each common components has a uniform gain is postponed to future work.

Third, the proposed algorithm can be viewed from a nonlinear filtering standpoint. By applying the proposed algorithm, we maintain or even enhance the common hidden variables, while filtering out the nuisance variables that are sensor/observation-specific.

While that the previous remarks addressed the proposed implementations, and the potential theoretical gaps that should be addressed in the future, the last remark deals with the practice of computing the approximation of the proposed implementation $d^{(\cup)}_t$. For the application of sleep stage identification described in Section \ref{sec:Sleep Results}, we have empirically found that a modified computation of $d^{(\cup)}_t$ gives rise to improve performance. This modification results in a ``smoother'' embedding, better representing the sleep stage.
In the alternative implementation, rather than calculating $d^{(\cup)}_t$ as in \eqref{eq:D_unify}, we calculate it in the following way.
First, for each pair of sensors we apply the standard \gls{DM} based on the pairwise kernels $\mathbf{K}^{(m)\bigcap(n)}$, $1\leq m,n \leq M, m \neq n$ computed in \eqref{eq:commom_ker_def}. For each pair we obtain a $L_{(m)\bigcap(n)}$-dimensional representation, where $L_{(m)\bigcap(n)}$ is a chosen parameter for the pair $(m,n)$, estimated using the ``spectral gap'' of the decay of the eigenvalues of $\mathbf{K}^{(m)\bigcap(n)}$.
Second, we concatenate the low-dimensional representations obtained from the previous step into a single vector. In other words, we now have $N$ concatenated $L$-dimensional vectors, where $L=\sum_{1\leq m,n \leq M, m \neq n}L_{(m)\bigcap(n)}$, representing the $N$ observations taken simultaneously from all $M$ sensors.
Third, we calculate the pairwise distance $d^{(\cup)}_t$ between the $N$ new concatenated vectors. Broadly, this technique is similar to \cite{keller_audio_visual_2010}, only here we combine the already ``filtered'' components (the results of \gls{AD} rather than \gls{DM}).
Since these vectors consist of components from different sensors, we chose to use a modified version of the Mahalanhobis distance. 
This modified Mahalanobis distance was first introduced in \cite{Coifman_Singer:2008}, and since then, was shown to exhibit remarkable capability to standardize measurements from different sources, e.g. in \cite{wu2015assess,TalmonPNAS,TalmonACHA,talmon2015manifold}. In \cite{TalmonPNAS,TalmonACHA}, it was shown to build intrinsic representations by revealing a hidden process driving the measurements. Recently, this technique was applied to multimodal data in \cite{salhov2016multi}. Importantly, compared with \gls{AD} and our proposed method, these methods \cite{TalmonPNAS,TalmonACHA,salhov2016multi} combine the information embodied in all the measurements and do not attempt to suppress nuisance variables or to extract only the common components.

The numerical implementation of the Mahalanobis distance deserves a remark.
The computation of the Mahalanobis distance requires estimation of the local covariance matrices of the vectors, each of size $L \times L$. This computation might be computationally cumbersome when $L$ is large, as often in our case.
In order to relax the required computational load, prior to the computation of the Mahalanobis distance, one can project the concatenated samples onto a lower dimensional vector space, for example, using \glspl{RP} \cite{candes2006near}, and then, compute the Mahalanobis distance for the projected samples with reduced dimensionality.
This heuristic method for calculating the common diffusion $d^{(\cup)}_t$ is summarized in Algorithm \ref{alg:SmoothUnifyScheme}.
\begin{algorithm}[t]
\caption{Mahalanobis-based Union Scheme}
\label{alg:SmoothUnifyScheme}
\textbf{Input}: $M(M-1)$ alternating-diffusion operators $\mathbf{K}^{(m)\bigcap(n)}$, $1\leq m,n\leq M, m \neq n$\\
\textbf{Output}: Alternating-diffusion distance $d^{(\cup)}_t$
\begin{enumerate}
\item  For $1\leq m,n\leq M, m \neq n$, calculate the spectral decomposition of each kernel $\mathbf{K}^{(m)\bigcap(n)}$, and obtain its eigenvalues $\left\{\lambda^{(m)\bigcap(n)}_l\right\}_{l=0}^{N-1}$ and eigenvectors $\left\{\boldsymbol{\psi}^{(m)\bigcap(n)}_l\right\}_{l=0}^{N-1}$.

\item For $1\leq m,n\leq M, m \neq n$, build an $L_{(m)\bigcap(n)}$-dimensional representation using standard \gls{DM} \eqref{eq:dm} for each time sample $i=1 \ldots N$.
    \begin{equation}
         \Psi^{(m)\bigcap(n)}_t(i) =  [\lambda_{1}^{t}\psi^{(m)\bigcap(n)}_1(i),\lambda_{2}^{t}\psi^{(m)\bigcap(n)}_2(i),\ldots,\lambda_{L_{(m)\bigcap(n)}}^{t}\psi^{(m)\bigcap(n)}_{L_{(m)\bigcap(n)}}(i)]
    \end{equation}
    where $t>0$ is a tuneable parameter.

\item For each time sample $i=1 \ldots N$, concatenate the low-dimensional representations into a single vector
    \begin{equation}
    \begin{split}
    \Psi^{(\cup)}_t(i) = \biggl(\Psi^{(1)\bigcap(2)}_t(i), \Psi^{(1)\bigcap(3)}_t(i),&\ldots,\Psi^{(1)\bigcap(M)}_t(i),\\
    \Psi^{(2)\bigcap(1)}_t(i),\Psi^{(2)\bigcap(3)}_t(i),&\ldots,\Psi^{(2)\bigcap(M)}_t(i),\\
    &\ldots\\
    \Psi^{(M)\bigcap(1)}_t(i),\Psi^{(M)\bigcap(2)}_t(i),&\ldots,\Psi^{(M)\bigcap(M-1)}_t(i) \biggl)
    \end{split}
    \end{equation}

\item Calculate $d_t^{(\cup)}$ using the Mahalanobis distance:
    \begin{equation}
    d^{(\cup)}_t(i,j)=\|\Psi^{(\cup)}_t(i)-\Psi^{(\cup)}_t(j)\|_\mathrm{Mahalanobis}
    \end{equation}
    for $i,j=1,\ldots,N$.
\end{enumerate}
\end{algorithm}

\begin{algorithm}[t]
\caption{Common Graph}
\label{alg:CG}

\textbf{Input}: Aligned samples from $M$ sets of observations: $\left\{(\boldsymbol{s}^{(1)}_i,\boldsymbol{s}^{(2)}_i,\ldots,\boldsymbol{s}^{(M)}_i)\right\}_{i=1}^N$.

\textbf{Output}: Low-dimensional representation of the common hidden random variables $\Theta$.

\begin{enumerate}
\item For each pair of observation sets $1\leq m,n\leq M, m \neq n$, apply alternating diffusion (Algorithm \ref{alg:AD}), and obtain the diffusion distance $d^{(m)\bigcap(n)}_{t}$.

\item Compute the distance $d^{(\cup)}_t$
\begin{equation}
    \label{eq:D_unify}
    d^{(\cup)}_t(i,j) =\sum_{1\leq m,n\leq M, m \neq n} d^{(m)\bigcap(n)}_t(i,j)
\end{equation}
for $i,j=1,\ldots,N$.

\item Based on the common diffusion distance $d^{(\cup)}_t$ calculate an affinity matrix
\begin{equation}
  W^{(\cup)}_{i,j}=\exp\left(-\frac{(d^{(\cup)}_t(i,j))^2}{\varepsilon^{(\cup)}}\right)
\end{equation}

\item Construct the diffusion operator $\mathbf{K^{(\cup)}}$:
\begin{equation}
    Q^{(\cup)}_{i,i}=\left(\sum\limits_{l=1}^{N} W^{(\cup)}_{i,l}\right)^{-1}  ;  \mathbf{K^{(\cup)}}= \mathbf{Q^{(\cup)}}\mathbf{W^{(\cup)}}
\end{equation}


\item  Apply standard diffusion maps (steps 3 and 4 in Algorithm \ref{alg:DM}) using $\mathbf{K^{(\cup)}}$, and obtain an $L$-dimensional representation of $\Theta$.
\end{enumerate}
\end{algorithm}

\section{Simulation Results}
\label{sec:Simulation Results}
Consider the toy problem described in Section \ref{ssec:toy_problem}. We simulate $6$ hidden scalar variables: $3$ common variables $\left( \theta^{(1)},\theta^{(2)},\theta^{(3)} \right)$ and $3$ nuisance variables $\left( n^{(1)},n^{(2)},n^{(3)} \right)$.
The variables are statistically independent and uniformly distributed in $[0,2\pi]$.
We then build $3$ sets of $N$ RGB images: $\{\boldsymbol{r}^{(1)}_i\},\{\boldsymbol{r}^{(2)}_i\}, \{\boldsymbol{r}^{(3)}_i\}, \ i=1,\ldots,N$.
The sensitivity table of this example is given by
\begin{equation}\label{toy_sensitivity}
 \mathbf{S}^T = \left(
	\begin{array}{ccc}
	1 & 1 & 0 \\
	0 & 1 & 1 \\
	1 & 0 & 1 \\
	\end{array}
	\right).
\end{equation}

Each image contains $3$ arrows, where each arrow is rotated according to a randomly generated angle: the angles of the arrows in $\boldsymbol{r}^{(1)}_i$ are $\left(\theta^{(1)}_i,\theta^{(2)}_i,n^{(1)}_i \right)$, the angles in $\boldsymbol{r}^{(2)}_i$ are $\left( \theta^{(2)}_i,\theta^{(3)}_i,n^{(2)}_i \right)$, and the angles in $\boldsymbol{r}^{(3)}_i$ are $\left(\theta^{(3)}_i,\theta^{(1)}_i,n^{(3)}_i \right)$.
The dimensionality of each RGB image is $36\times96\times3$. We column-stack the RGB images, i.e., $\boldsymbol{r}^{(1)}_i,\boldsymbol{r}^{(2)}_i,\boldsymbol{r}^{(3)}_i $ are vectors of length $J=10368$.
The proposed algorithm is data-driven, and therefore, it does not assume any prior knowledge on the nature of observations.
In order to highlight this important property, we use \glspl{RP}.  
First, \glspl{RP} with sufficiently large dimension maintain the underlying geometry, yet the image appearances are lost,
which shows that our algorithm does not apply any image processing. Second, in the original images, the different hidden variables are manifested in separate coordinates/pixels; \glspl{RP} mix the hidden variables, enabling a more challenging extraction task.
We generate $D=1600$ orthonormal vectors $\{ \boldsymbol{b}_i\}_{i=1}^{D}$ of length $J$ and denote by $\mathbf{B} \in \mathbb{R}^{J \times D}$ the matrix whose columns are these random vectors.
We build the data of the sensors (cameras) by \glspl{RP} ${\boldsymbol{s}}^{(m)}_i= \mathbf{B}^T \boldsymbol{r}^{(m)}_i$, where $m$ is the camera index. In the case of data acquired by cameras, $\mathbf{B}$ can be viewed as the coding system in the cameras.
An illustration of the images and their \glspl{RP} is depicted in Figure~\ref{fig:RP_process}.
Illustration of the ``movies'' of the \glspl{RP} captured by each camera can be seen in the following link \href{https://youtu.be/91N6mhlYQYY}{https://youtu.be/91N6mhlYQYY}.

\begin{figure}[t]
\centering
\includegraphics[scale=0.35]{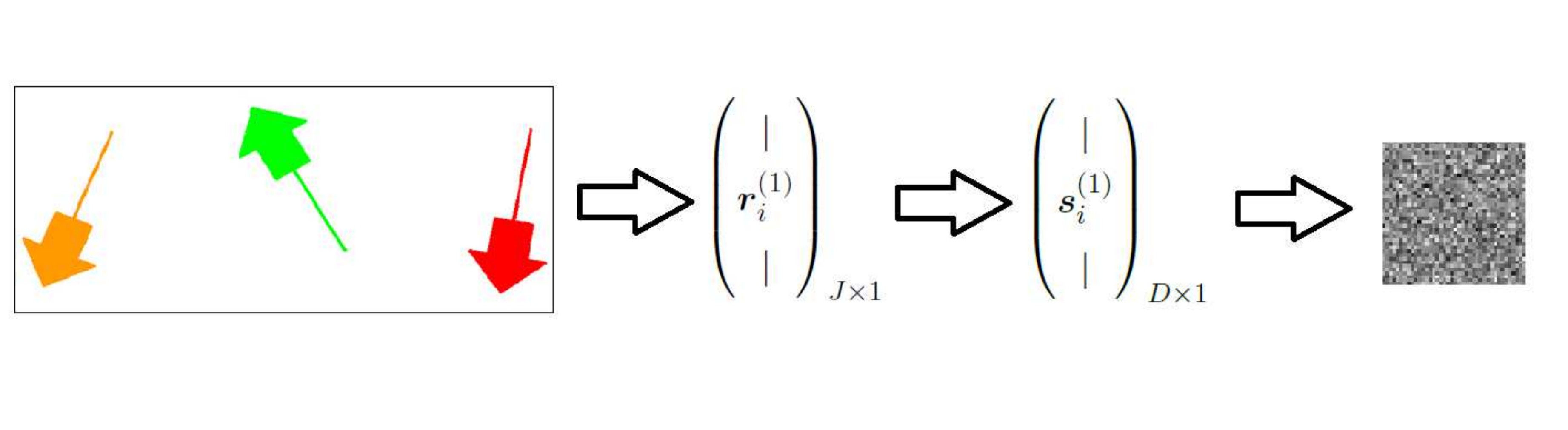}
\caption{Random projection diagram of the $i$th image. Each RGB image was column stacked into a vector of length $J=10368$. Then it was projected on a subspace of $\mathds{R}^{D}$ using an orthonormal set $\{v_i\}_{i=1}^{D}$. The projection is illustrated by a gray-scale $40\times40$ image. As can be seen the image's property are lost through this projection.}
\label{fig:RP_process}
\end{figure}

We first apply \gls{DM} separately to each set of observations. Figure~\ref{fig:DM} presents $2$-dimensional views of the obtained $3$-dimensional embeddings.
Each subfigure presents a scatter plot of embedded data-points. Each data-point is an image (a frame in the movie) captured by a certain camera after a random-projection ${\boldsymbol{s}}^{(m)}_i$, where $i$ is the frame index and $m$ is the camera index. The axes of the scatter plot are the first $3$ components of the obtained embedding derived from the corresponding camera. The embedded data-points are colored according to the rotating angles $\left( \theta^{(1)},\theta^{(2)},\theta^{(3)} \right)$ and $3$ noise variables $\left( n^{(1)},n^{(2)},n^{(3)} \right)$. It should be noted that this information (the color) was added after calculating the embedding and was not taken into account in the computation of embedding.
The subfigure in the $l$th column and in the $m$th row contains the embedded data-points derived from the $m$th camera $\{\boldsymbol{s}^{(m)}_i\}_{i=1}^{N}$, and its data-points are colored according to the rotating angle of the $l$th arrow.
In the $3$ left columns the color coding is according to  $\{\theta^{(1)}_i\}_{i=1}^{N}$, $\{\theta^{(2)}_i\}_{i=1}^{N}$, $\left\{\theta^{(3)}_i\right\}_{i=1}^{N}$, and in the $3$ right columns the color coding is according to $\{n^{(1)}_i\}_{i=1}^{N}$, $\{n^{(2)}_i\}_{i=1}^{N}$, $\{n^{(3)}_i\}_{i=1}^{N}$. In other words, in each row the same scatter plot is shown, but with different color coding. The $3$-dimensional scatter plots are rotated so that the obtained color gradient is best visualized from our $2$-dimensional view point.
For example, the subfigures in the second row are derived from the observations from the second camera $\{\boldsymbol{s}^{(2)}_i\}_{i=1}^{N}$ . The data-points in the first column are colored according to the rotation angles $\{\theta^{(1)}_i\}_{i=1}^{N}$, in the second column according to $\{\theta^{(2)}_i\}_{i=1}^{N}$, etc.

As can be seen, in each row, $3$ scatter plots exhibit a smooth color gradient, $2$ from the left $3$ columns and $1$ from the right $3$ columns, corresponding to the variables sensed by the respective camera. In the $3$ left columns, we see that the color gradients indicates accurate detection of the common variables according to the sensing matrix $\mathbf{S}$. On the $3$ right columns, only in the diagonal subfigures exhibit a smooth color gradient, indicating that each captures only its own nuisance variable, as expected.
In conclusion, Figure~\ref{fig:DM} implies that the obtained embeddings by \gls{DM} provide accurate parametrizations of the hidden variables measured by each observation (camera), both the common and the nuisance variables.

\begin{figure}[t]
\centering
\begin{tabular}{cccc:ccc}
&  $\theta^{(1)}$ &   $\theta^{(2)}$ &   $\theta^{(3)}$ &   $n^{(1)}$ &   $n^{(2)}$ &   $n^{(3)}$\\
$\boldsymbol{s}^{(1)}$ &
\includegraphics[scale=0.1]{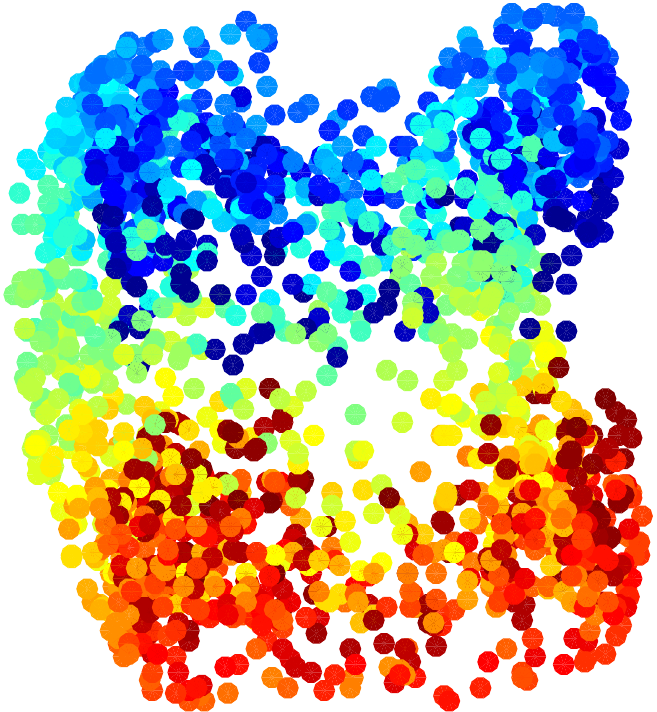} &
\includegraphics[scale=0.1]{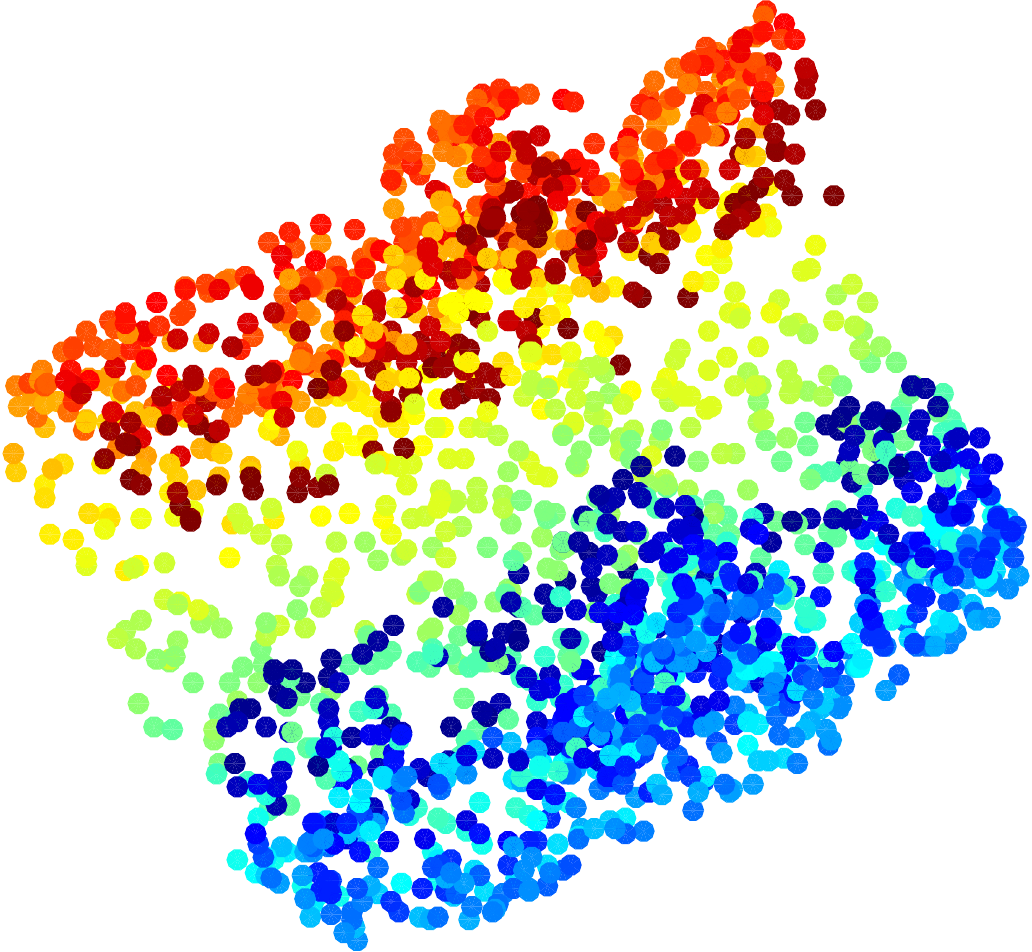} & \includegraphics[scale=0.1]{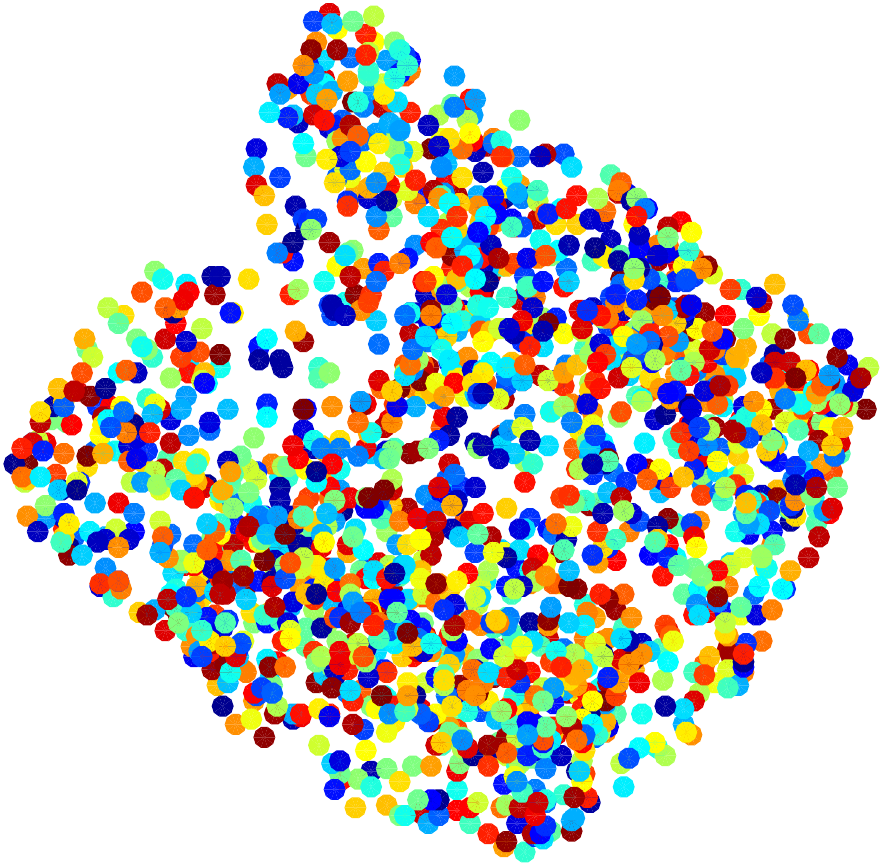} &\includegraphics[scale=0.1]{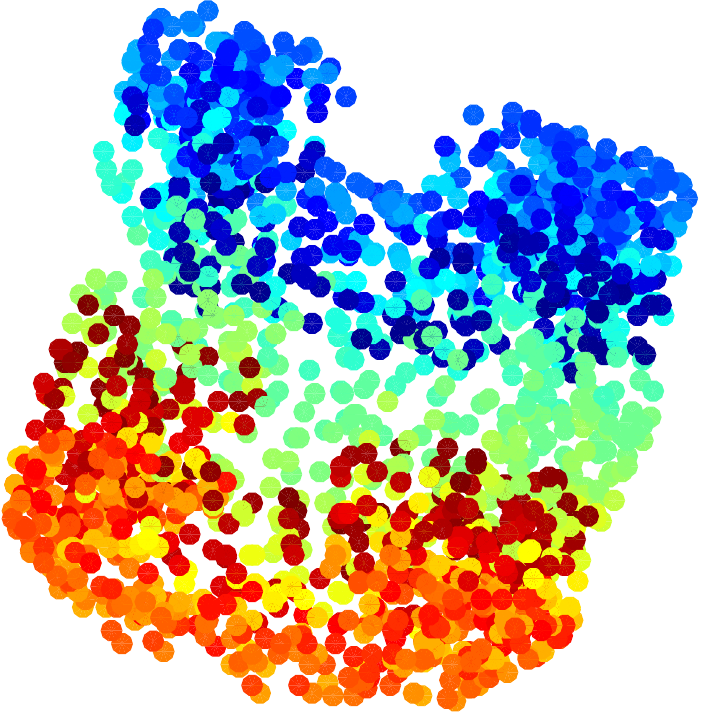} &
\includegraphics[scale=0.1]{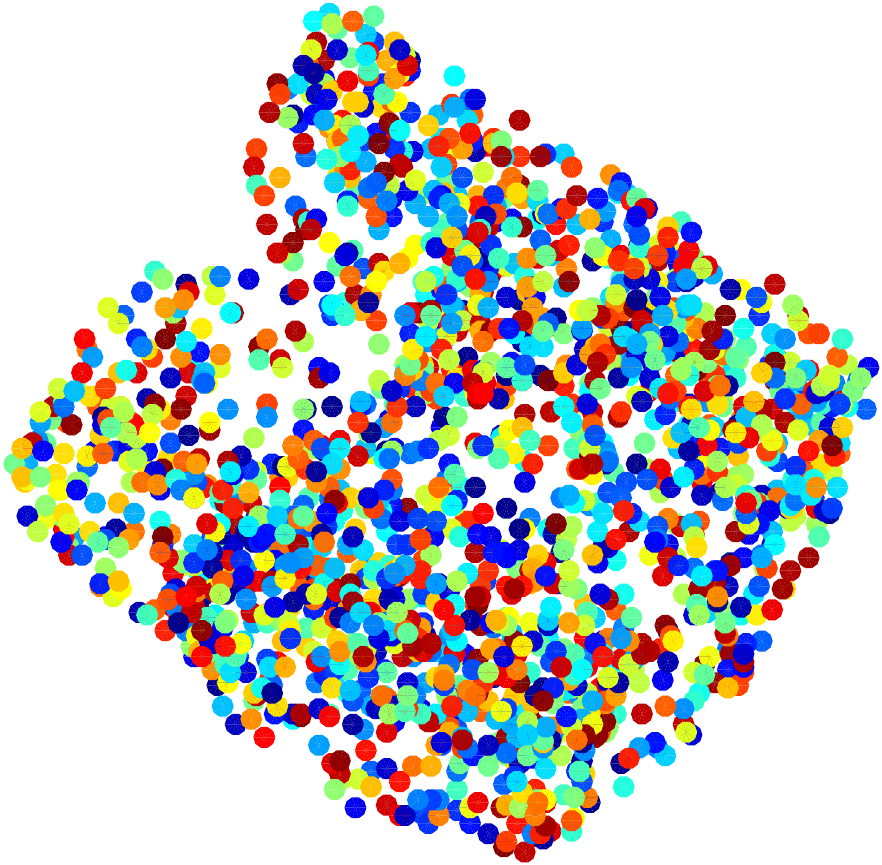} & \includegraphics[scale=0.1]{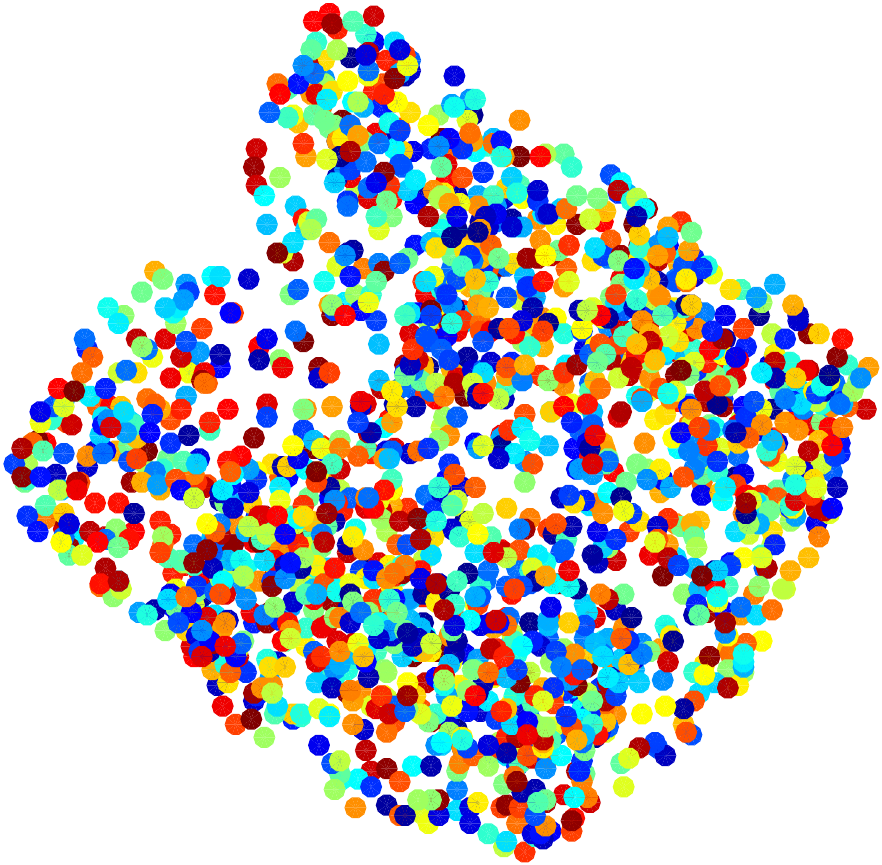}\\
$\boldsymbol{s}^{(2)}$ &
\includegraphics[scale=0.1]{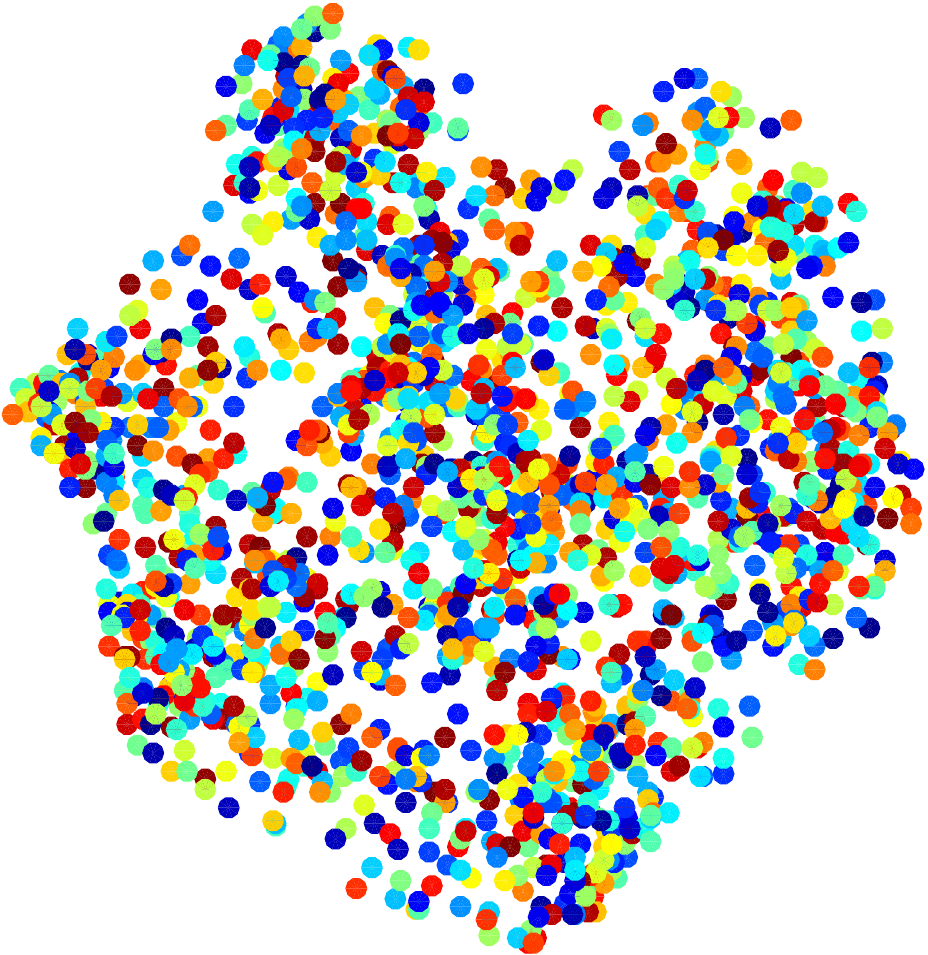} &
\includegraphics[scale=0.1]{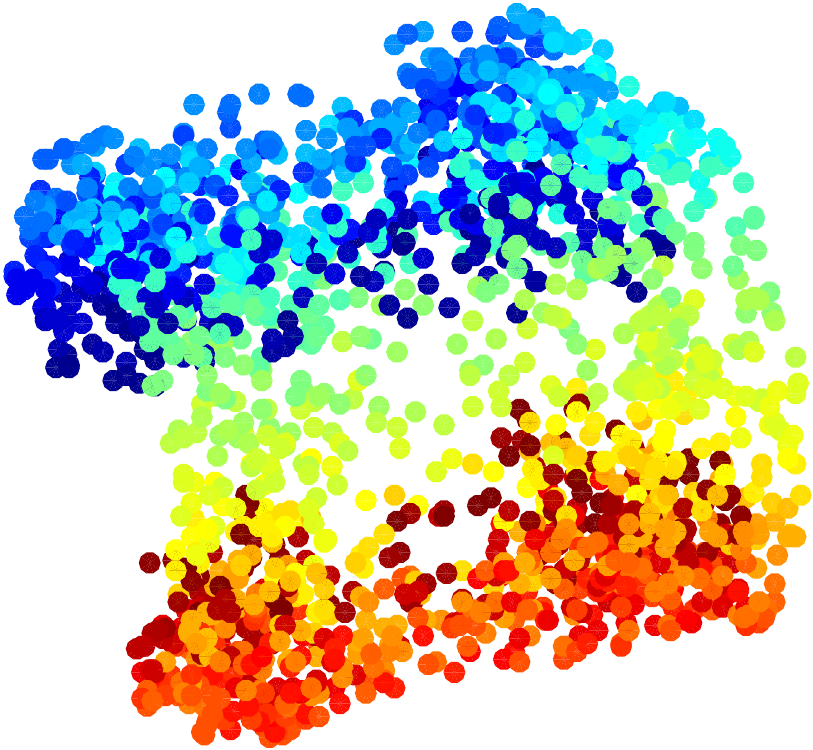} & \includegraphics[scale=0.1]{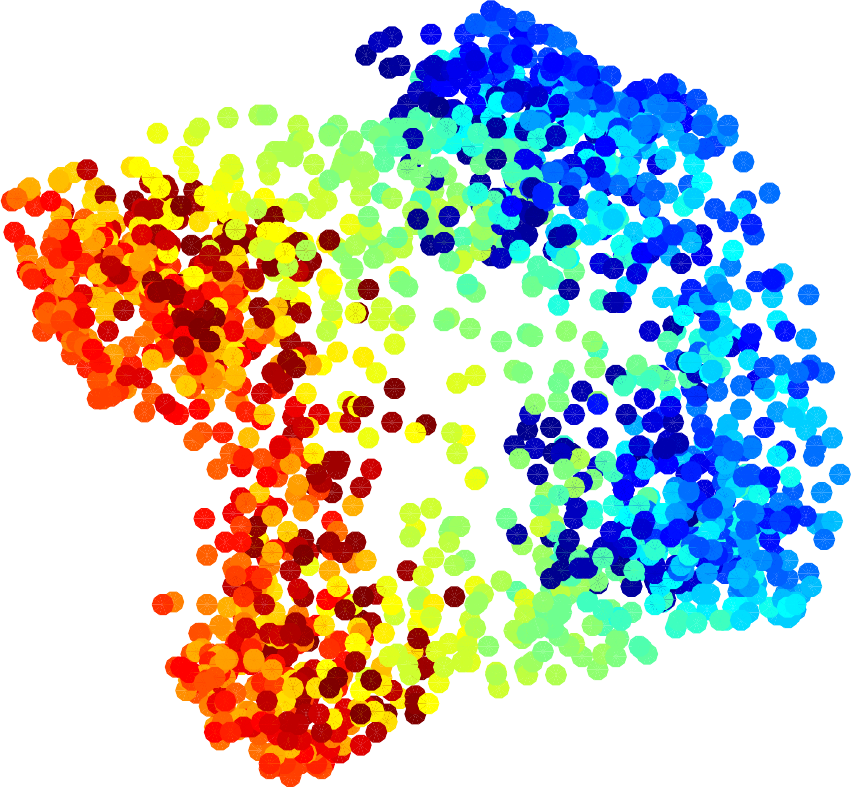} &\includegraphics[scale=0.1]{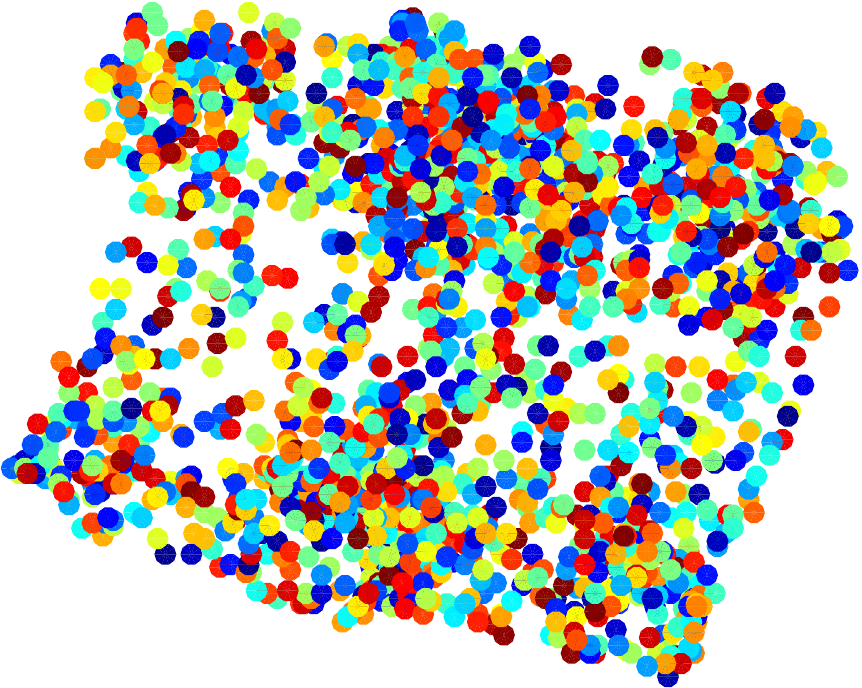} &
\includegraphics[scale=0.1]{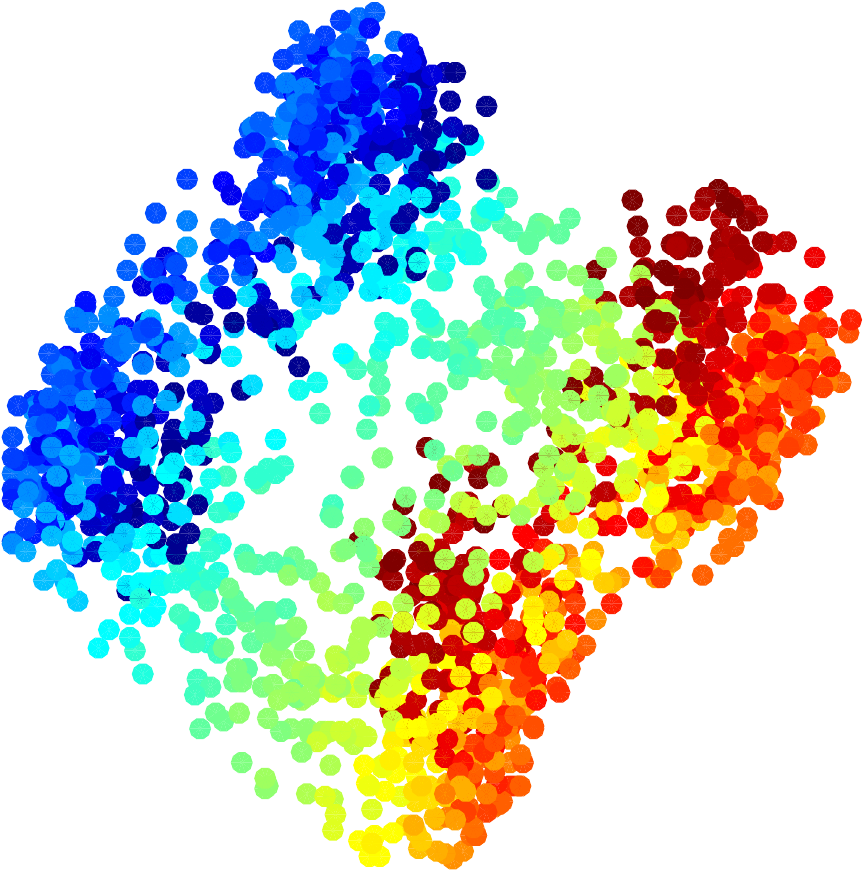} & \includegraphics[scale=0.1]{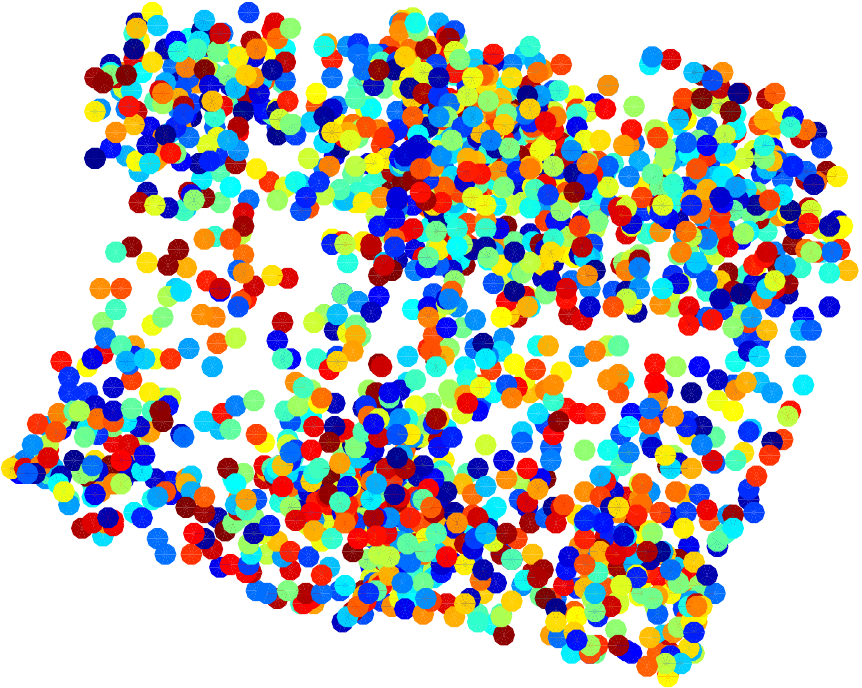}\\
$\boldsymbol{s}^{(3)}$ &
\includegraphics[scale=0.1]{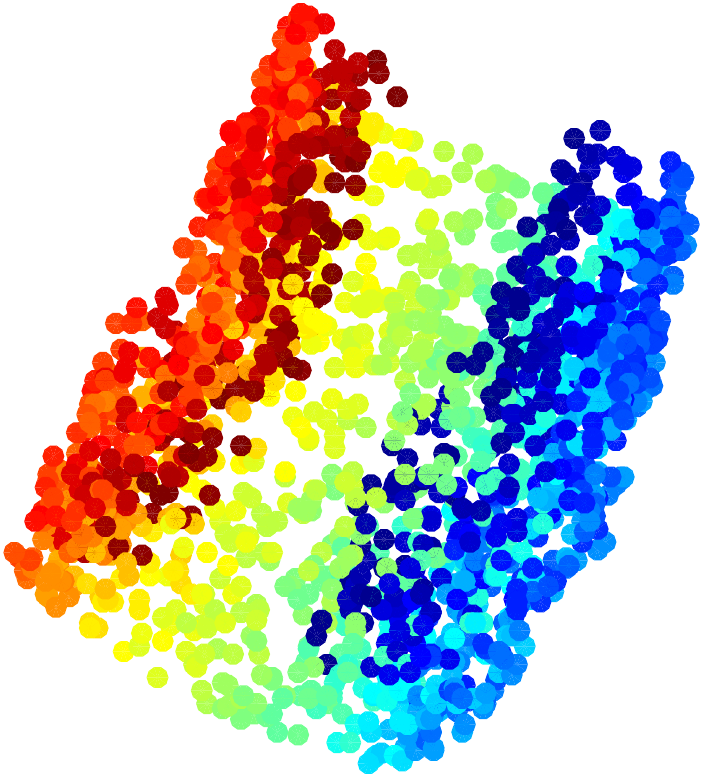} &
\includegraphics[scale=0.1]{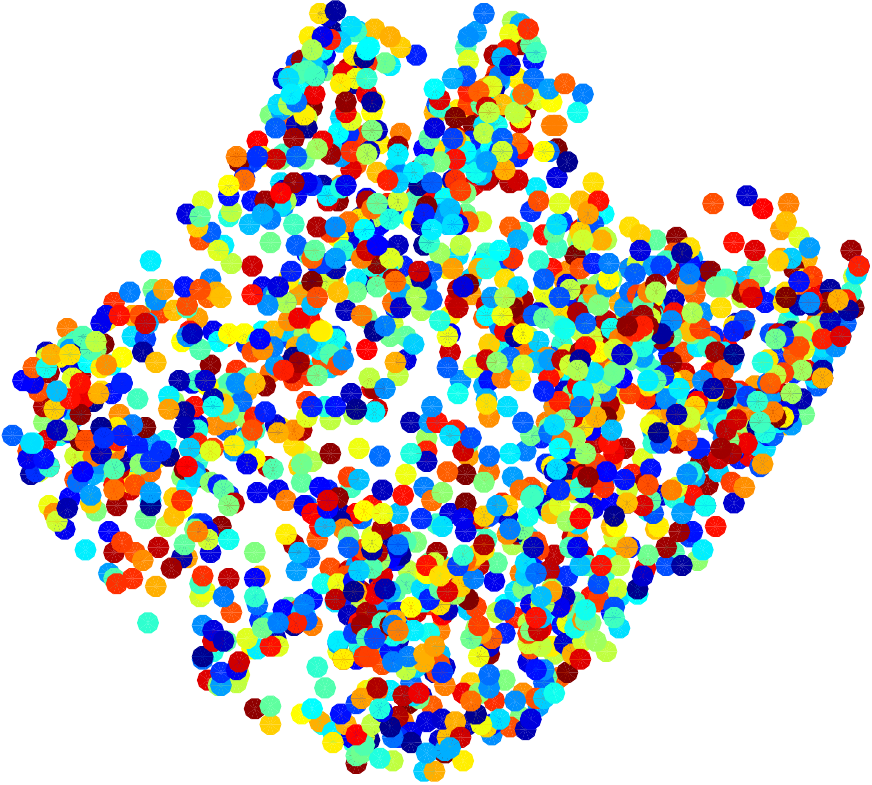} & \includegraphics[scale=0.1]{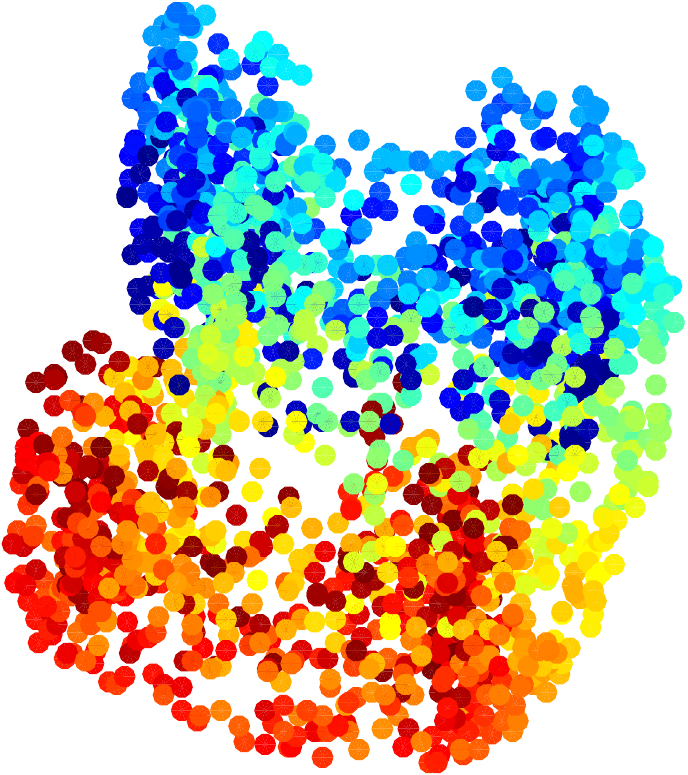} &\includegraphics[scale=0.1]{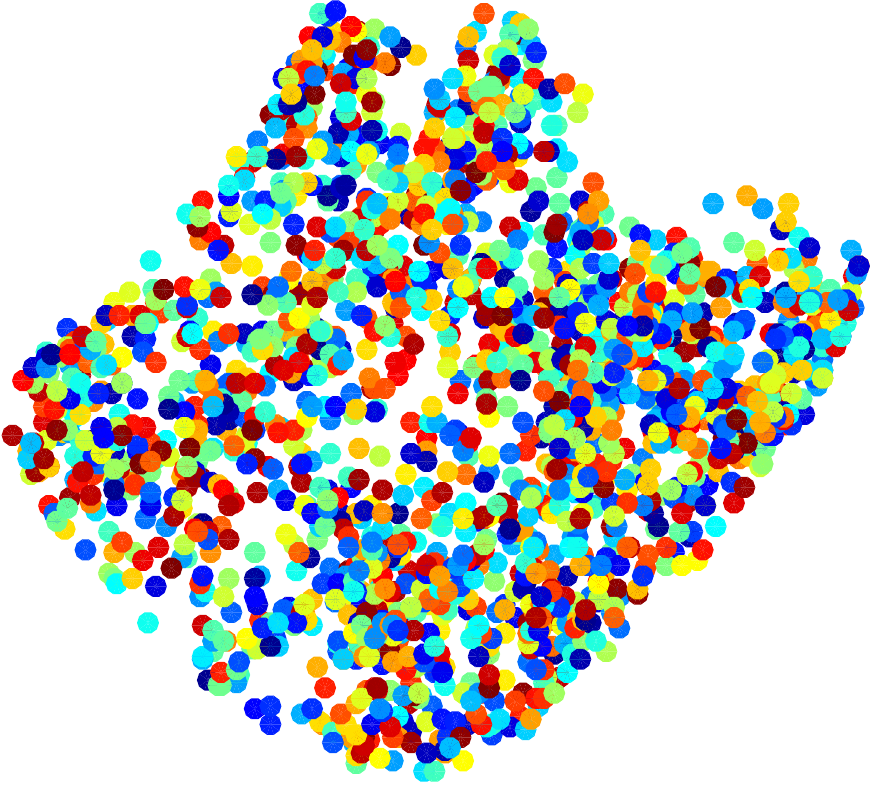} &
\includegraphics[scale=0.1]{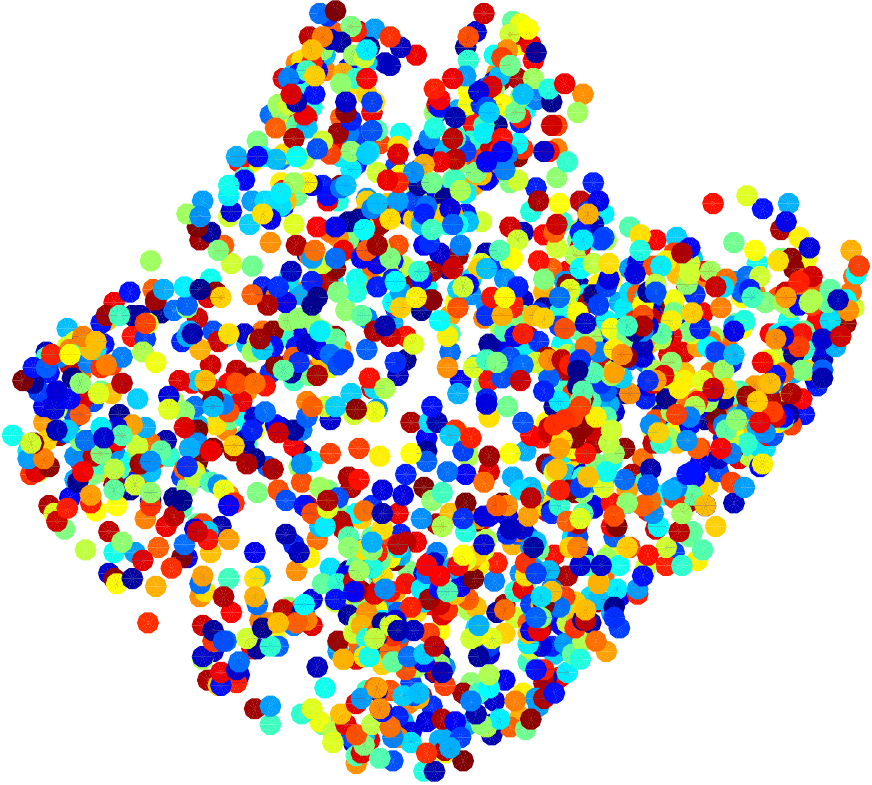} & \includegraphics[scale=0.1]{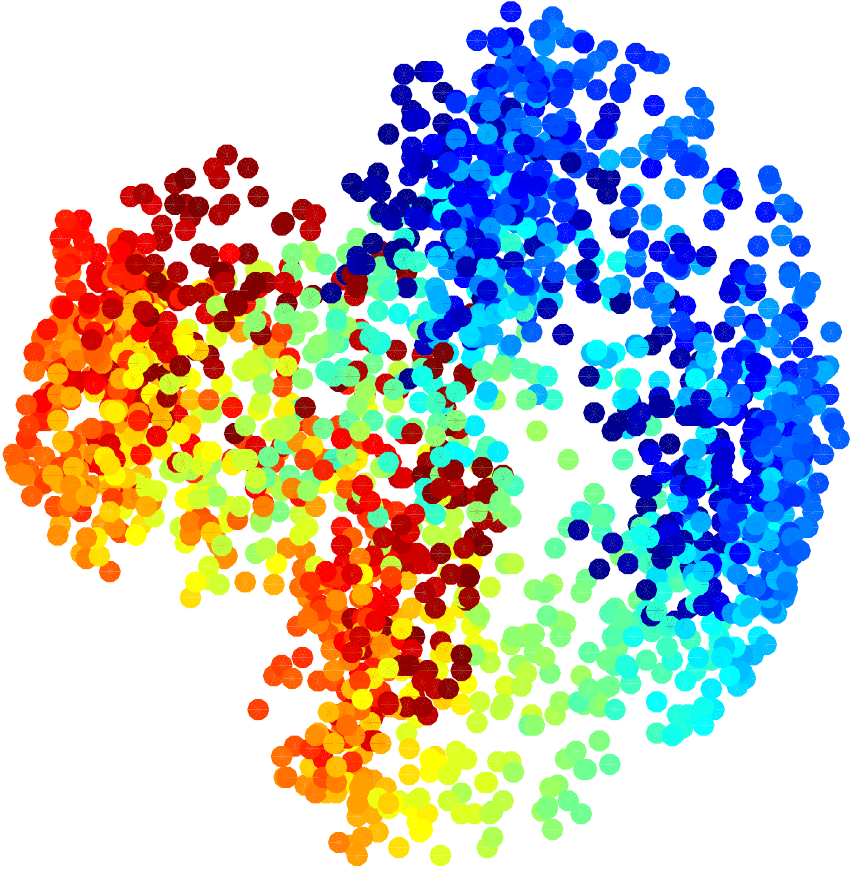}\\
\end{tabular}
\caption{3D embedding obtained by applying diffusion map on a single observer. The subfigures are arranged such that subfigures in each row are obtained from the same observer. The data-points in each column are colored according to different arrow's rotation angles. }
\label{fig:DM}
\end{figure}

The proposed algorithm is applied to the three sets of observations. The obtained embedding is depicted in Figure~\ref{fig:SDM}. The same $3$ dimensional scatter plot of the obtained embedding is shown with different color coding. The subfigures in the top row are colored (from left to right) according to the common variables $\left\{\theta^{(1)}_i\right\}_{i=1}^{N}$, $\left\{\theta^{(2)}_i\right\}_{i=1}^{N}$, $\left\{\theta^{(3)}_i\right\}_{i=1}^{N}$, while the subfigures in the bottom row are colored (from left to right) according to the nuisance variables $\left\{n^{(1)}_i\right\}_{i=1}^{N}$, $\left\{n^{(2)}_i\right\}_{i=1}^{N}$, $\left\{n^{(3)}_i\right\}_{i=1}^{N}$. As in Figure \ref{fig:RP_process}, the $3$ dimensional embedding is rotated, such that the corresponding color gradient is emphasized from the depicted $2$ dimensional point of view.
We can see from the obtained color gradients that the embedding provides a parametrization of only the common variables, meaning that the proposed algorithm manages to extract all $3$ of the common variables $\left( \theta^{(1)},\theta^{(2)},\theta^{(3)} \right)$ (despite having none in common to all $three$ observations), while suppressing all $3$ nuisance observation-specific variables $\left( n^{(1)},n^{(2)},n^{(3)} \right)$.
Upon publication, the Matlab code and data of this toy problem will be made available online.

\begin{figure}[t]
\centering
\begin{tabular}{ccc}
  $\theta^{(1)}$ &   $\theta^{(2)}$ &   $\theta^{(3)}$\\
\hspace{-.4in} \includegraphics[scale=0.2]{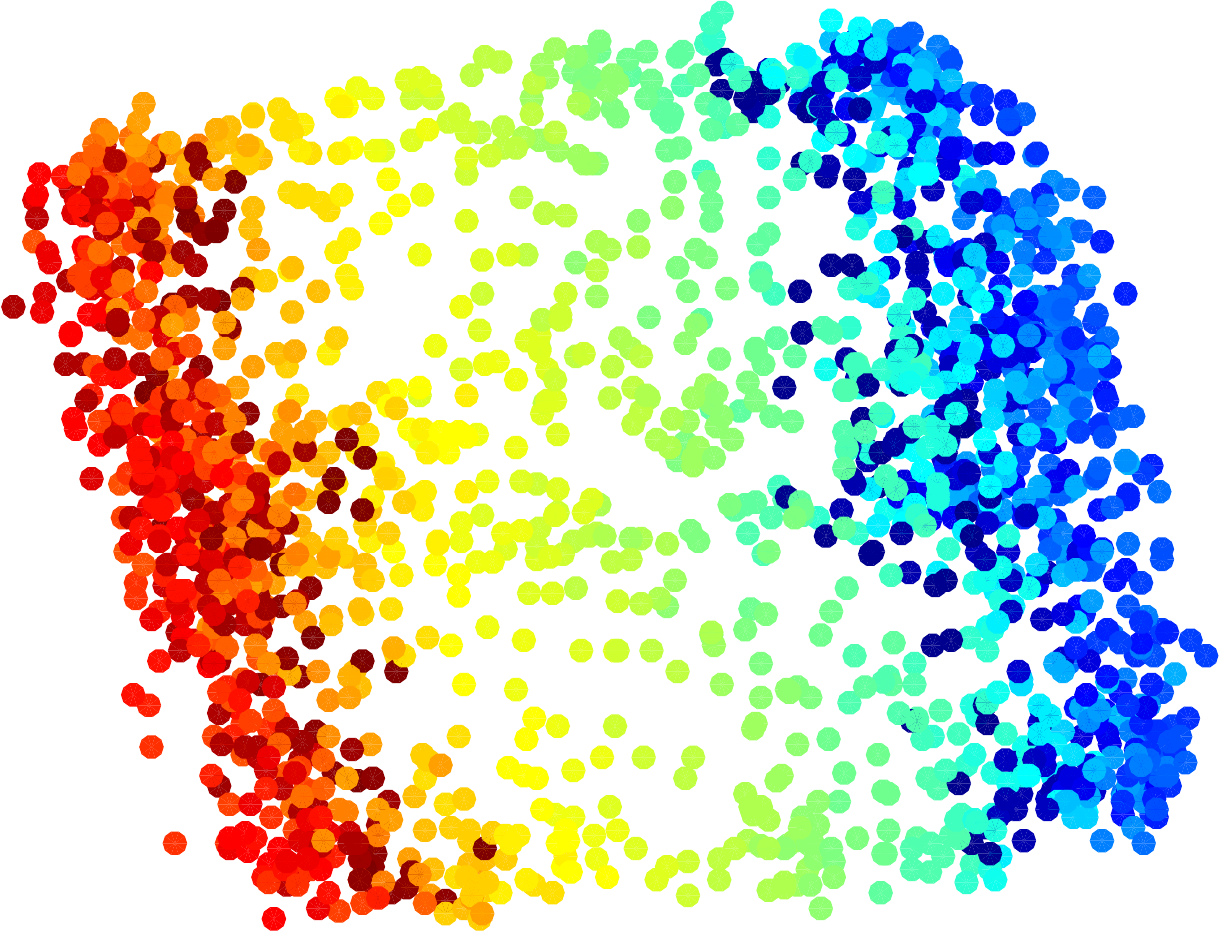} &
\includegraphics[scale=0.2]{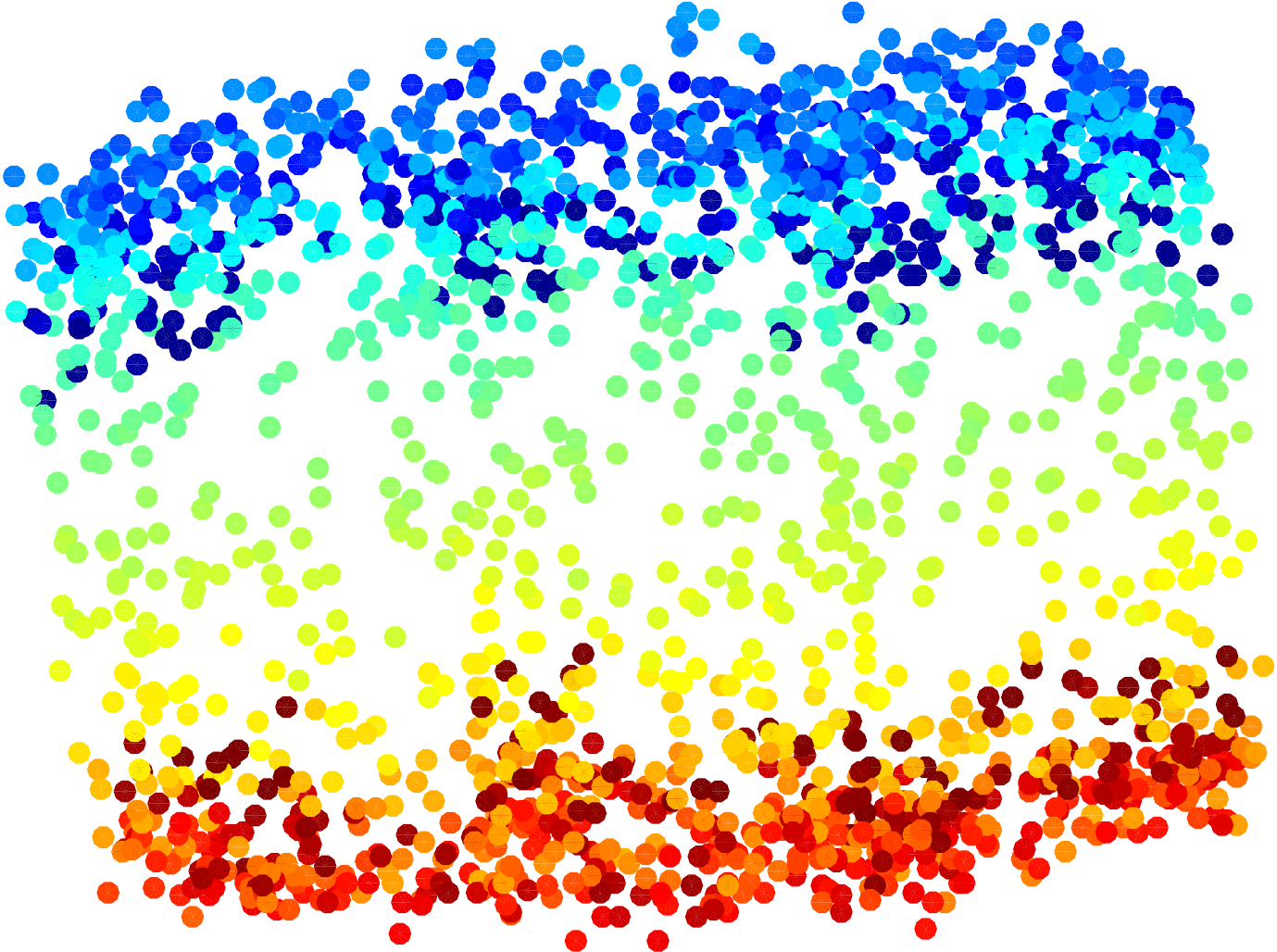} & \includegraphics[scale=0.2]{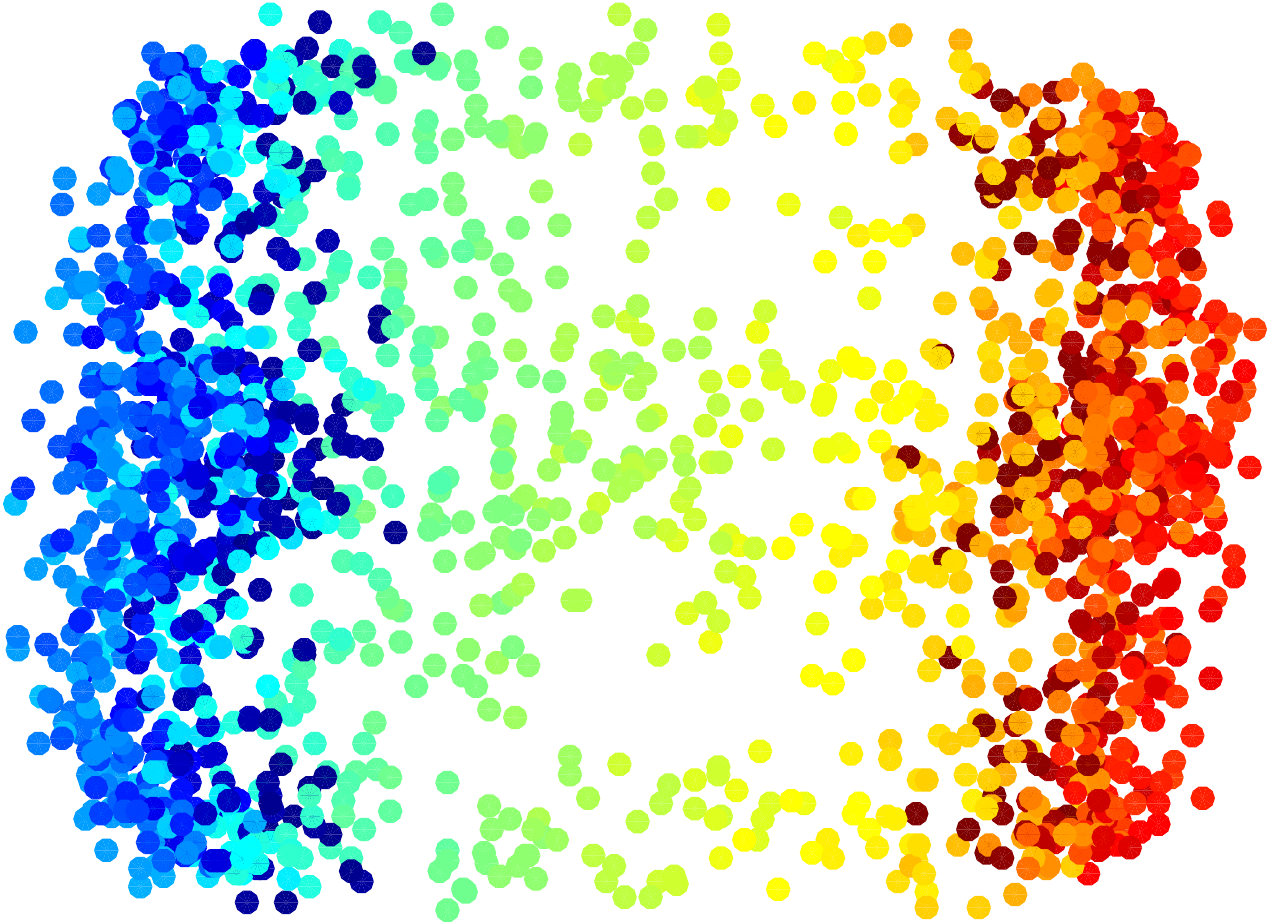}\\  \hdashline{}  
\hspace{-.4in} \includegraphics[scale=0.2]{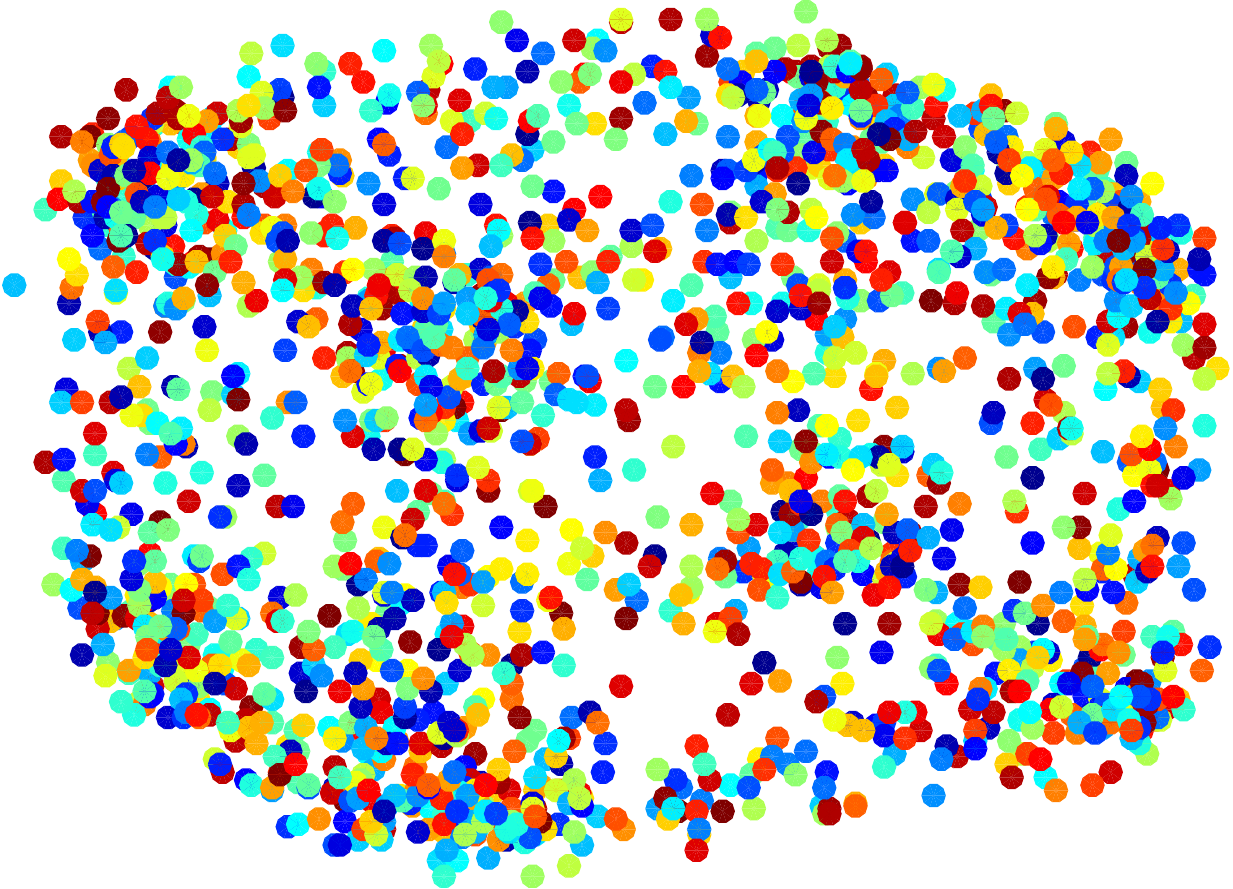} &
\includegraphics[scale=0.2]{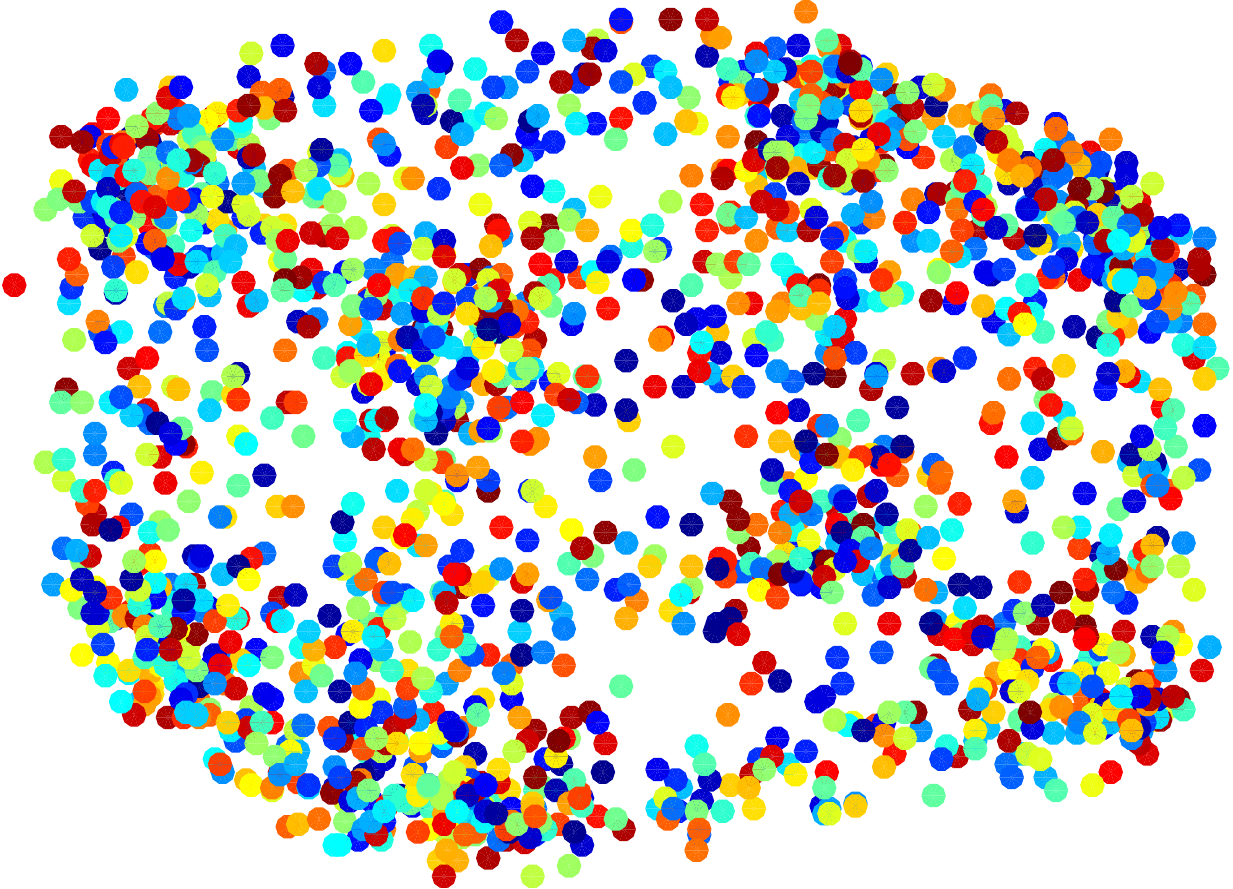} & \includegraphics[scale=0.2]{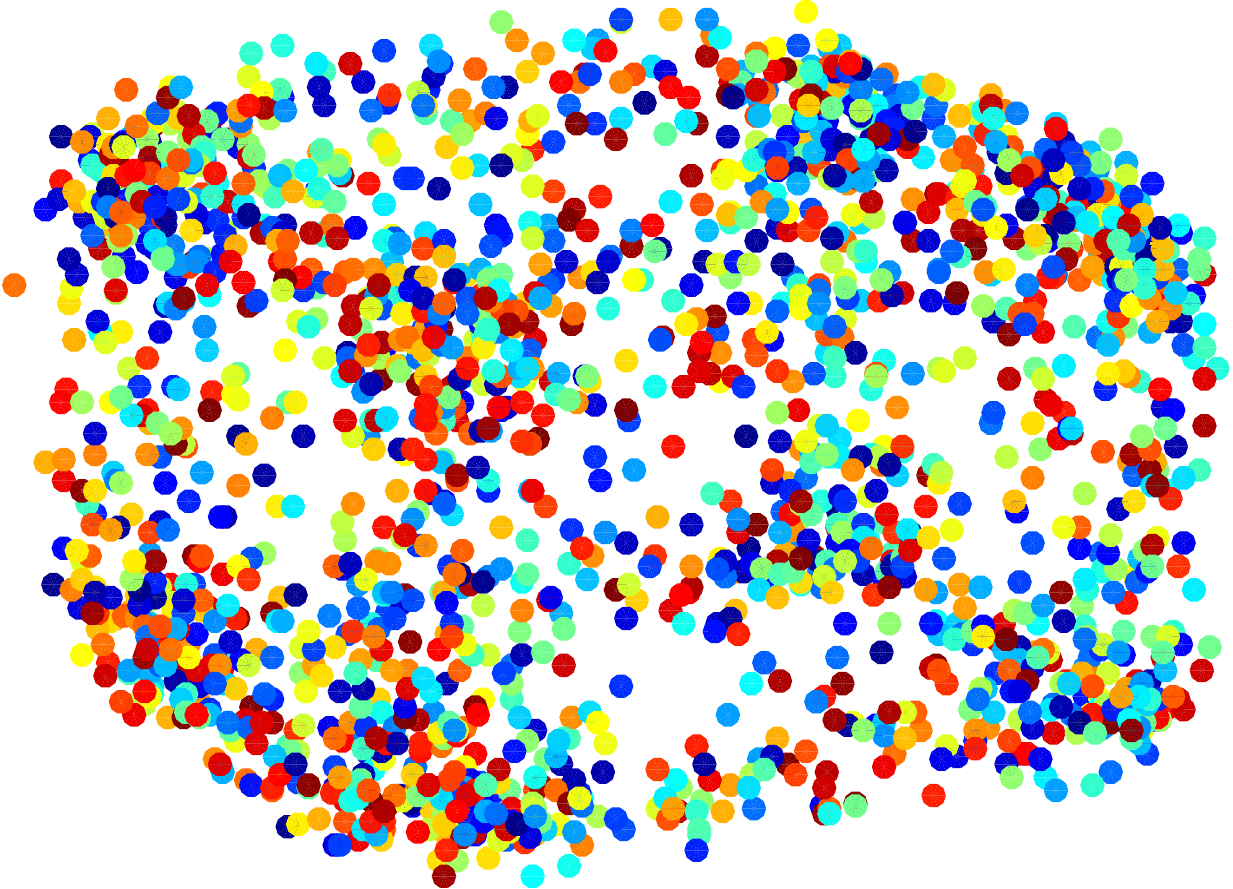}\\
  $n^{(1)}$ &   $n^{(2)}$ &   $n^{(3)}$
\end{tabular}
\caption{3D embedding obtained by applying the proposed algorithm on the observers set. The subplots in the first rows are colored according to the common variables, the subplots in the second row are colored according to the noise variables. As can be seen, the obtained parametrization corresponds to the common variables.}
\label{fig:SDM}
\end{figure}

\section{Application to Sleep Stage Assessment}
\label{sec:Sleep Results}
As mentioned above, the problem of extracting the common hidden variables from multiple data sets taken by different observables can be perceived as a problem of nonlinear filtering. To demonstrate the potential of this particular nonlinear filtering scheme in processing real data, we apply the proposed algorithm to sleep data, where the ultimate goal is to devise an automatic system for sleep stage assessment.

Sleep is a global and recurrent physiological process, which is in charge of the memory consolidation, the learning redistribution, tissue regeneration, immune system enhancement, etc \cite{Lee-Chiong:2008}. The sleep dynamics are characterized by particular temporal physiological features, which are intimately related to the quality of sleep.
The clinically acceptable sleep stage is mainly determined by reading recorded electroencephalogram (EEG) signals based on the Rechtschaffen and Kales (R\&K) criteria \cite{Rechtschaffen_Kales:1968,Iber_Ancoli-Isreal_Chesson_Quan:2007}. In the R\&K criteria, the sleep dynamics are divided into two broad stages: rapid eye movement (REM), and non-rapid eye movement (NREM) \cite{Lee-Chiong:2008}.
The NREM stage is further divided into two shallow stages, which are denoted N1 and N2, and a deep sleep stage, which is denoted N3.
In addition to the interest stemming from physiological aspects, sleep stage assessment has important clinical applications. For example,
REM is associated with perceptual skill improvement \cite{Karni1994}, slow wave sleep is associated with Alzheimer's disease \cite{Kang2009}, poor sleep quality is associated with weaning failure \cite{RocheCampo2010}, etc.
Besides personal health purposes, the sleep quality is also responsible for several public catastrophes \cite{Colten_Altevogt:2006}. These facts indicate the importance of an accurate automatic annotation system for sleep stage assessment and its broad applications.

In the past decades, various automatic annotation methods have been proposed. Those methods mainly extract various features from the EEG recordings for the purpose of studying sleep dynamics \cite{Bajaj_Pachori:2013}, such as time domain summary statistics, spectral or coherence features, time-frequency features, and information entropy, just to name a few \cite{Kannathal_Choo_Acharya_Sadasivan:2005,Blanco_Quiroga_Rosso_Kochen:1995,Geng_Zhou_Yuan_Cai_Zeng:2011}. Recently, a theoretically solid approach suitable for analyzing and estimating the dynamics of the brain activity from recorded EEG signals has been proposed in \cite{TalmonPNAS,TalmonACHA}.
%
A particular aspect of sleep dynamics, which has not gained much research attention in the line of research mentioned above, is that sleep is not localized solely in the brain and is reflected in other physiological systems as well. 
%
For example, the regulation of mechanoreceptor and the chemoreceptor leads to breathing pattern variability in the respiratory signal. We have a remarkably regular breathing during N3 stage and irregular breathing with fast varying instantaneous frequency and amplitude during REM stage. Those physiological phenomena motivated various studies to explore the relation between the sleep stage and the patterns in the respiratory signals, e.g. \cite{Chung_Choi_Kim_Lim_Choi_Jeong_Park:2007,Guerrero-Mora_Elvia_Bianchi_Kortelainen:2012,Sloboda_Das:2011}.
Physiologically, these variations are not originated from the same controller, and phenomenologically do not have the same patterns in the recorded time series. Thus, while we could observe the sleep dynamics via observing the characteristics of different sensors, each of them reflects only part of the sleep dynamic, and is complicated by the nature of the sensor.

Based on the above physiological facts, an automatic approach for assessing the sleep stage was presented in \cite{wu2015assess}. It relies on the assumption that there exist hidden low-dimensional physiological processes driving the sleep dynamics, and hence the accessible measured signals. However, these hidden processes may be deformed by the observation procedures; each observation (e.g., an EEG channel measuring brain activity or a chest belt measuring respiration) can be influenced by nuisance factors, which are sensor- or channel-specific (e.g., the specific type of sensors and their exact positions), yet our interest is in the intrinsic variables related to the sleep stages. In \cite{wu2015assess}, empirical intrinsic geometry (EIG) method \cite{TalmonPNAS,TalmonACHA}, which is based on nonlinear independent component analysis \cite{Coifman_Singer:2008} and was proven to be invariant to the measurement modality, was applied to build an intrinsic representation of the measured data.
In \cite{lederman2015icassp}, this method was extended to a pair of sensors. It was shown that by analyzing the measurements taken simultaneously from $2$ sensors, a more reliable intrinsic representation of the sleep dynamics can be obtained, compared with the analysis based only on a single signal.

In this section we extend the algorithm shown in \cite{lederman2015icassp}, and process jointly multiple channels. We show that extracting the underlying common variables from multiple data sets acquired in different channels recovers systematically a representation, which is well correlated with the sleep stage.
The analogy to the setting described in this paper is as follows.
We assume that the sleep dynamics are intimately related to hidden controllers that affect the respiratory as well as the brain neural system. These controllers are not accessible to us; yet, they can be recovered by analyzing observations from multiple channels/sensors, each captures different, partial yet complementary aspects of it.
Under this assumption, our interest is in obtaining the intrinsic variables underlying the measurements related to these controllers.
On the one hand, by analyzing multiple observation channels we can gather more information on the hidden controllers.
On the other hand, observations from each channel might be deformed by the different acquisition and measurement modalities and may be affected by noise and interferences, specific to the particular (type of) sensor.
In the context of this work, this tradeoff is addressed by defining the intrinsic variables (related to the hidden controllers of interest) as those which are not sensor-specific, and hence, the variables of interest are those that are common among at least  two observables.

Twenty subjects without sleep apnea were chosen for this study. The demographic characteristics of these individuals fall within the normal ranges. We used recordings of $6$ hours per subject, which were performed in the sleep center at Chang Gung Memorial Hospital (CGMH), Linkou, Taoyuan, Taiwan. The institutional review board of the CGMH approved the study protocol (No. 101-4968A3) and the enrolled subjects provided written informed consent. See \cite{wu2015assess} for more details regarding the experimental setting and the collected data.

We build the common graph according to Algorithm \ref{alg:CG} for extracting the common hidden variables separately to two sets of sensors.
The first set includes $3$ signals: abdominal and chest motions, which are recorded by piezo-electric bands, and airflow, which is measured using thermistors and nasal pressure, all $3$ at sampling rate of $100$~Hz. The second set comprises recordings from $4$ EEG channels: C3A2, C4A1, O1A2 and O2A1 at sampling rate of $200$~Hz.
The recorded respiratory signals are denoted by $R_m, m=1,\ldots,3$ and the EEG signals are denoted by $E_m, m=1,\ldots,4$.

Prior to the application of our method, each of the single-channel recordings was preprocessed by applying the scattering transform as in \cite{wu2015assess}, which was shown to improve the regularity and stability of signals with respect to various deformations \cite{Mallat2012}.
We then apply Algorithm \ref{alg:CG} separately twice: once to the respiratory set, and once to the EEG set.

In order to demonstrate the inherent ``sensor selection'' capability of the proposed method, for each set of measurements we added an artificial ``pure noise'' sensor to simulate possible sensor failure.
Because our processing pipeline begins with the application of the scattering transform, which automatically degenerates any stationary noise, the noise sensor consists of a non-stationary sequence generated by modulating a sine-wave according to
\begin{align}
\label{eq:noise_sensor}
  \phi(\tau)&=\frac{1}{2}+\frac{1}{4}\sin\left(\frac{2\pi \tau}{512\cdot10}\right) \\
  n(t)&=\sin\left(2\pi \int_0^t \phi(\tau)d\tau\right) \nonumber
\end{align}
where $n(t)$ is the continuous time signal and $\phi(\tau)$ can be viewed as the instantaneous frequency of $n(t)$.
We sample the obtained modulated sine-wave $n(t)$ at a sampling rates of $200$~Hz and $100$~Hz for the EEG set and for the respiratory set, respectively.
It should be noted that this particular non-stationary ``noise'' implementation was chosen just for the sake of demonstration, and any other non-stationary sequence could be chosen instead.

We compare the results of the common graph algorithm, analyzing multiple sensors, with the results attained by the standard \gls{DM} applied separately to each individual sensor.
In addition, we compare the results to two competing schemes analyzing multiple sensors.
In the first scheme, we concatenate the scattering transform components from each sensor, and then, apply the standard \gls{DM}.
We note that conceptually this scheme takes into account the information captured by all the sensors without any filtering.
We refer to the first scheme as the \emph{concatenation scheme}.
In the second scheme, we apply \gls{AD} to the entire set of sensors. Namely, we calculate the diffusion kernel $\mathbf{K}^{(m)}$ for the $m$th sensor, where $m=1\ldots M$ and build an \gls{AD} kernel based on the product of all the kernels, that is, $\mathbf{K}=\mathbf{K}^{(1)}\mathbf{K}^{(2)}\cdots \mathbf{K}^{(M)}$. Then, we apply \gls{DM} with this \gls{AD} kernel.
This scheme takes into account only the information that is captured simultaneously by all of the sensors, namely $\bigcap_{m \neq n}\left(\mathcal{S}^{(m)}\bigcap\mathcal{S}^{(n)}\right)$ , thereby performing excessive filtering.
We refer to the second scheme as the \emph{multiplication scheme}.

The calculation of the affinity matrices, which is a core element in the tested methods, is carried out using the Mahalanobis distance variant presented in \cite{Coifman_Singer:2008}, which was discussed in Section \ref{subsec:common_graph}.
To be able to depict information embodied in more than three eigenvectors, we randomly project the embeddings attained by the competing algorithms to $3$ dimensions. This allows us to visually inspect the portion of the relevant information and the portion of the nuisance information manifested in the representations obtained by the different algorithms. We use the same projection in all tested methods. 

The \glspl{RP} of the embeddings are depicted in Figure \ref{fig:sleep_embed}. The \glspl{RP} based on the single channel \gls{DM} applied to the O2A1 EEG channel and to the airflow channel are depicted in the top row. The \glspl{RP} based on the concatenation scheme, multiplication scheme and the proposed algorithm are depicted in the second, third and bottom rows, respectively. The embeddings depicted on the left column are based on the EEG set, and on the right column are based on the respiratory set.
Each embedded point is colored according to its respective sleep stage, as identified by a human expert. Importantly, the information on the sleep stage (e.g., the color) was not taken into account in the algorithms forming the embeddings.

\begin{figure} [t] \centering
    \subfloat[]
    {\includegraphics[scale=0.25]{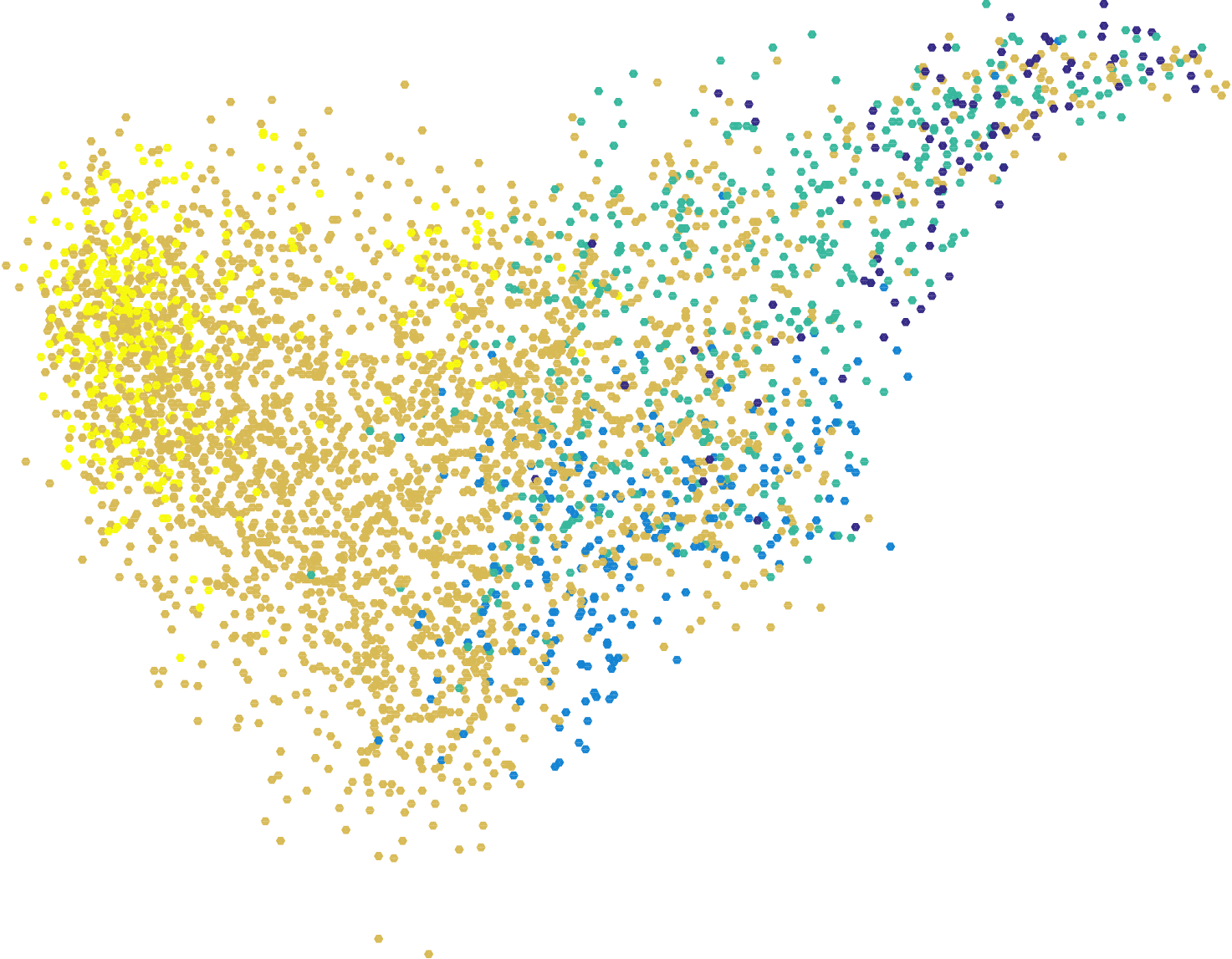}
    \label{fig:sleep_embed_DM_eeg}}
    \hspace{1cm}
    \subfloat[]
    {\includegraphics[scale=0.25]{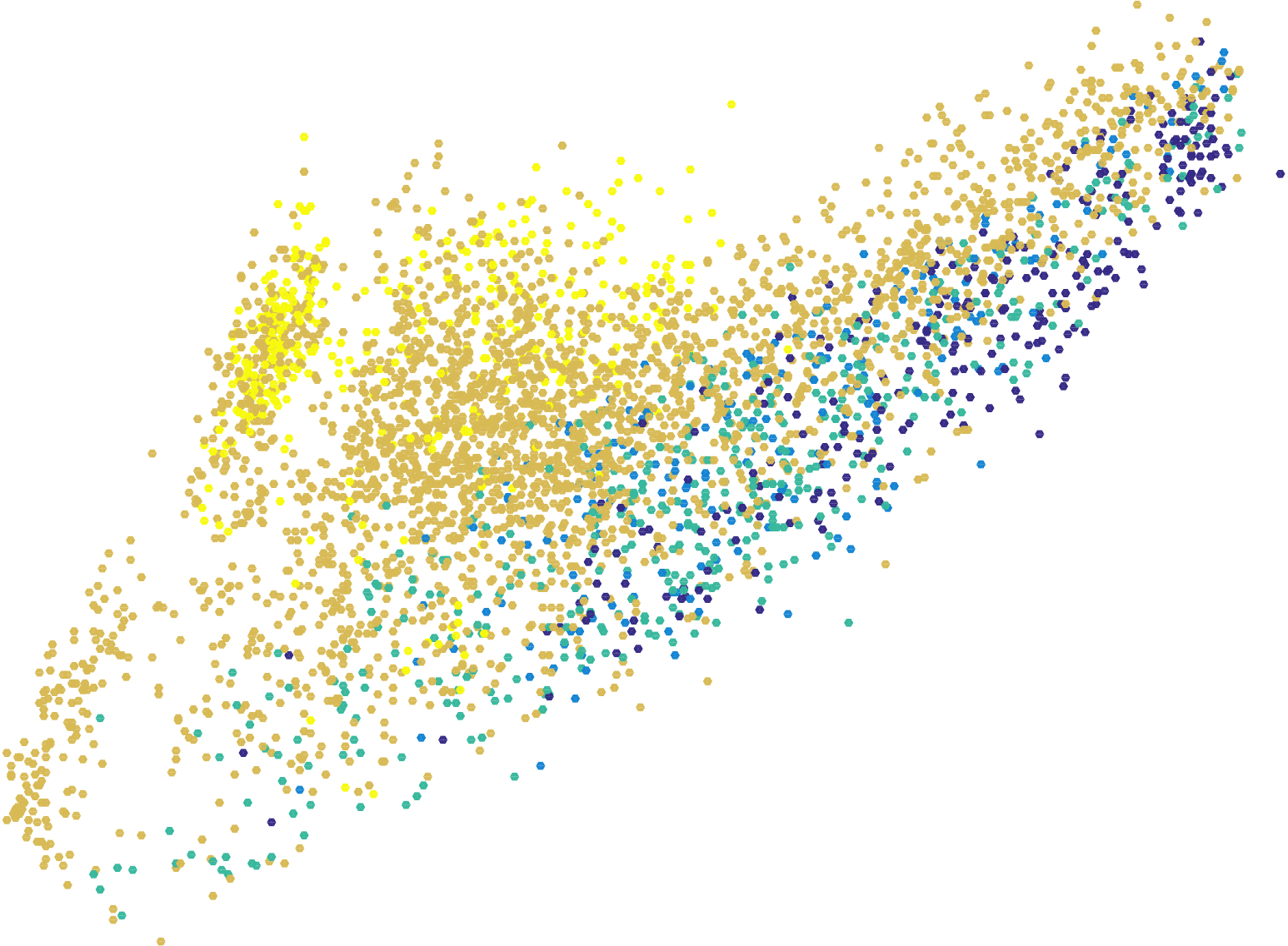}\label{fig:sleep_embed_DM_resp}}
    \\
    \subfloat[]
    {\includegraphics[scale=0.25]{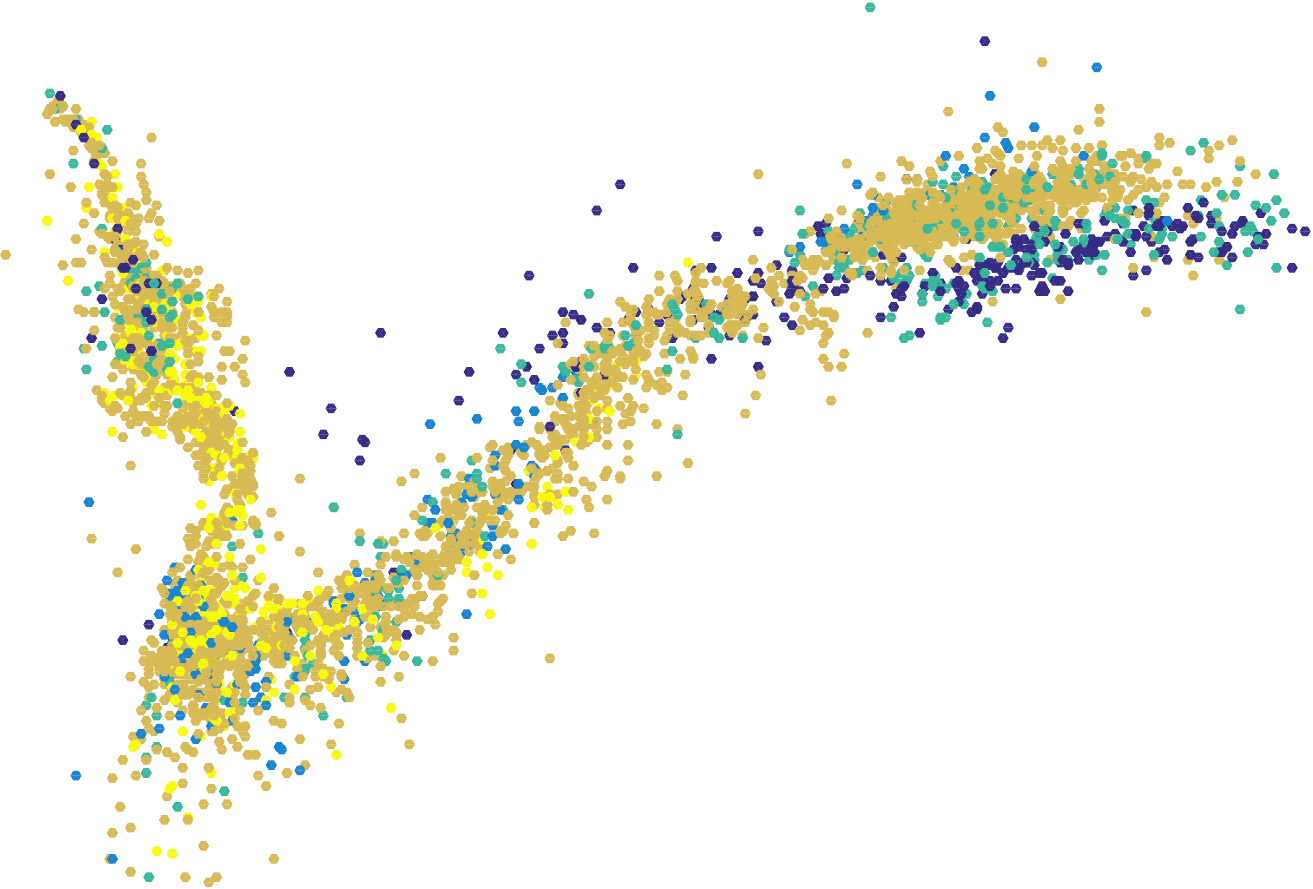}\label{fig:sleep_embed_concat_eeg}}
    \hspace{1cm}
    \subfloat[]
    {\includegraphics[scale=0.25]{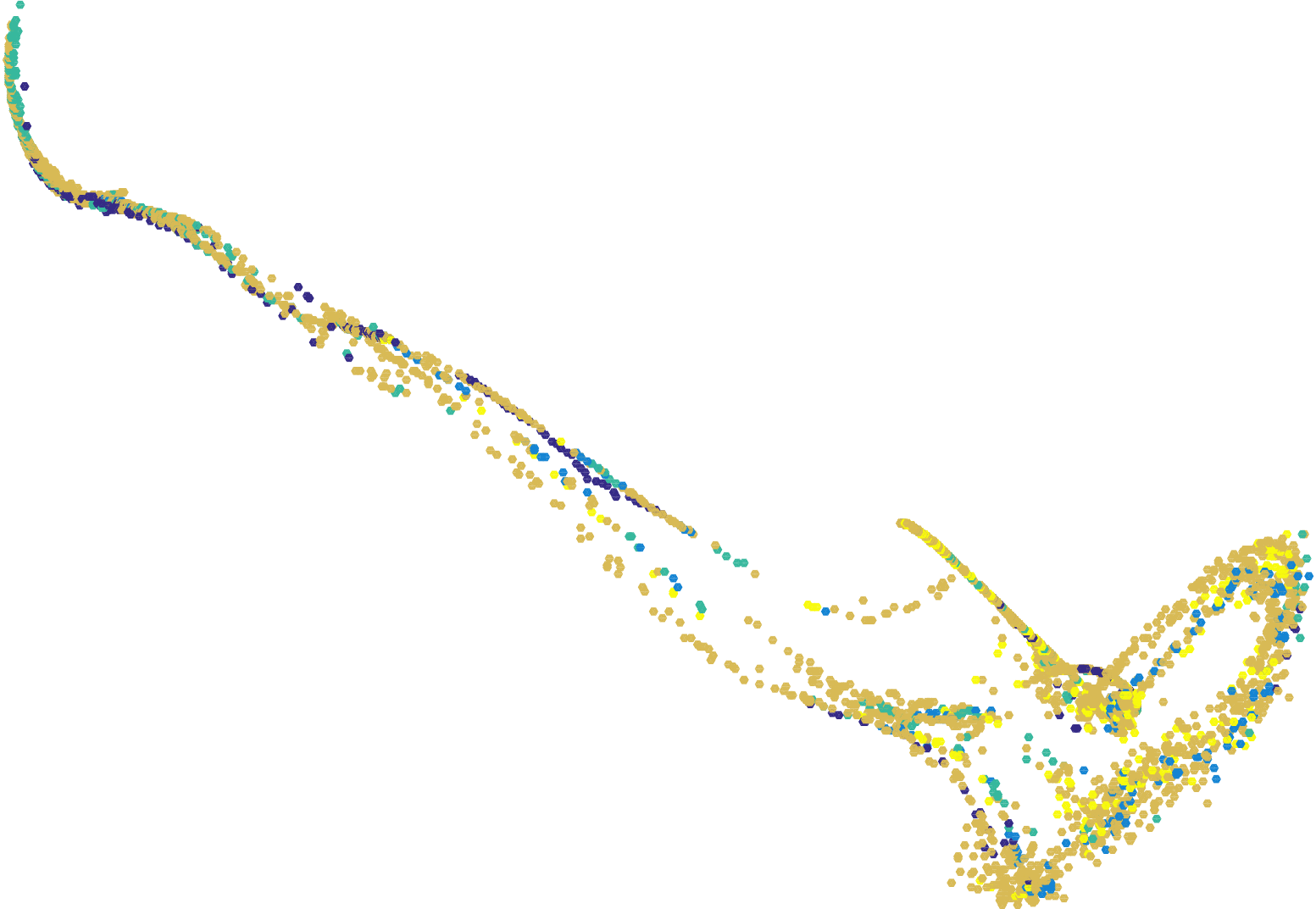}\label{fig:sleep_embed_concat_resp}}
    \\
    \subfloat[]
    {\includegraphics[scale=0.25]{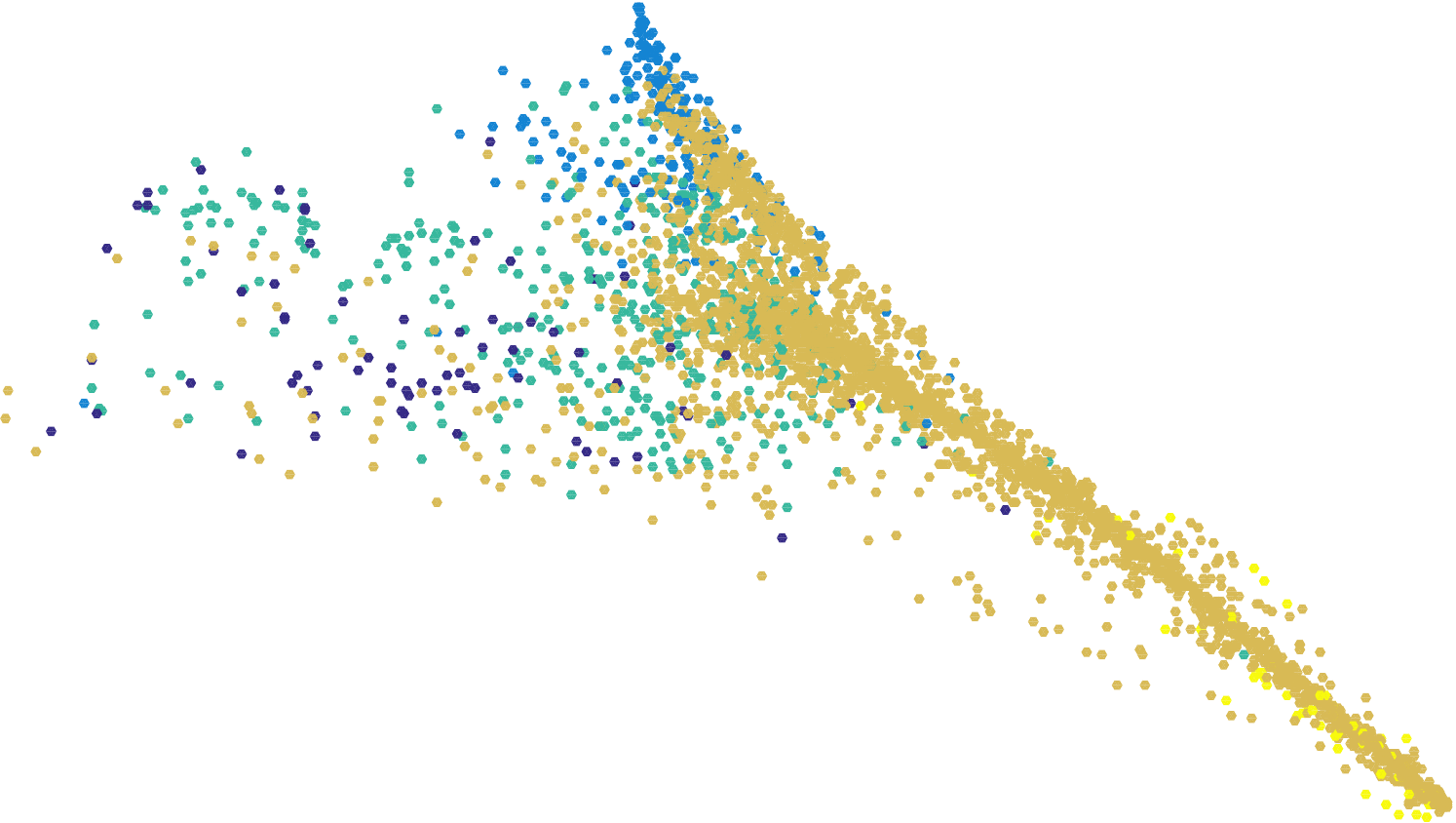}\label{fig:sleep_embed_mul_eeg}}
    \hspace{1cm}
    \subfloat[]
    {\includegraphics[scale=0.25]{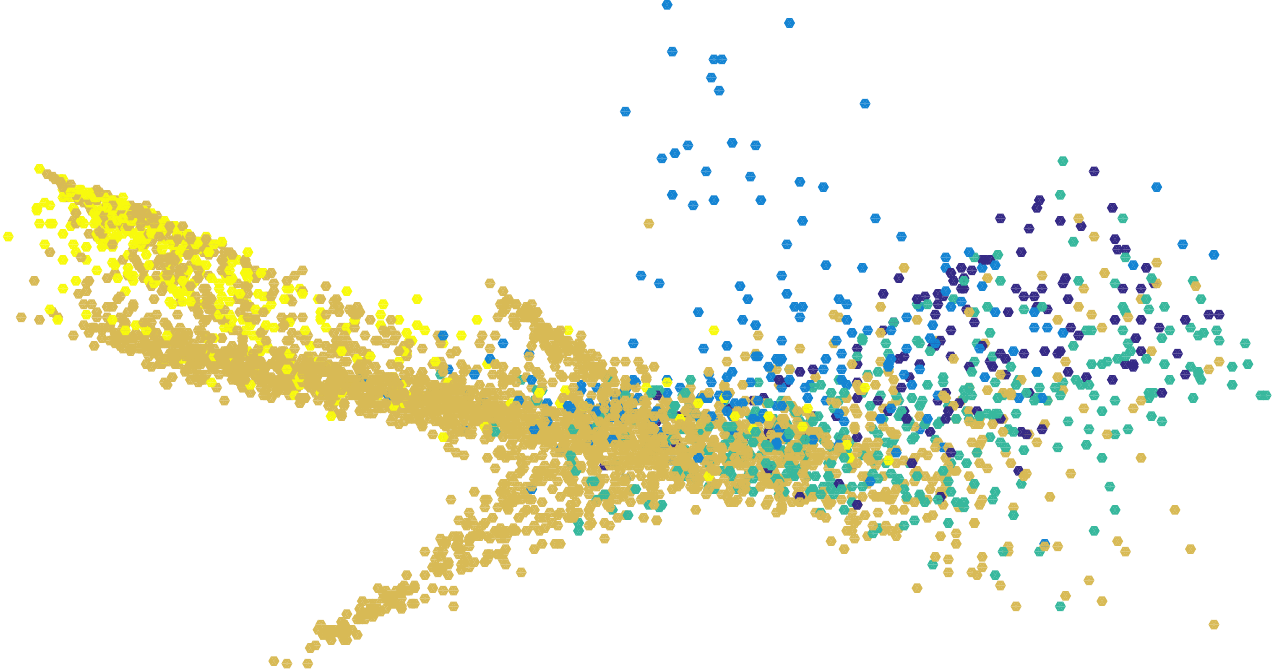}\label{fig:sleep_embed_mul_resp}}
    \hspace{1cm}

    \subfloat[]
    {\includegraphics[scale=0.25]{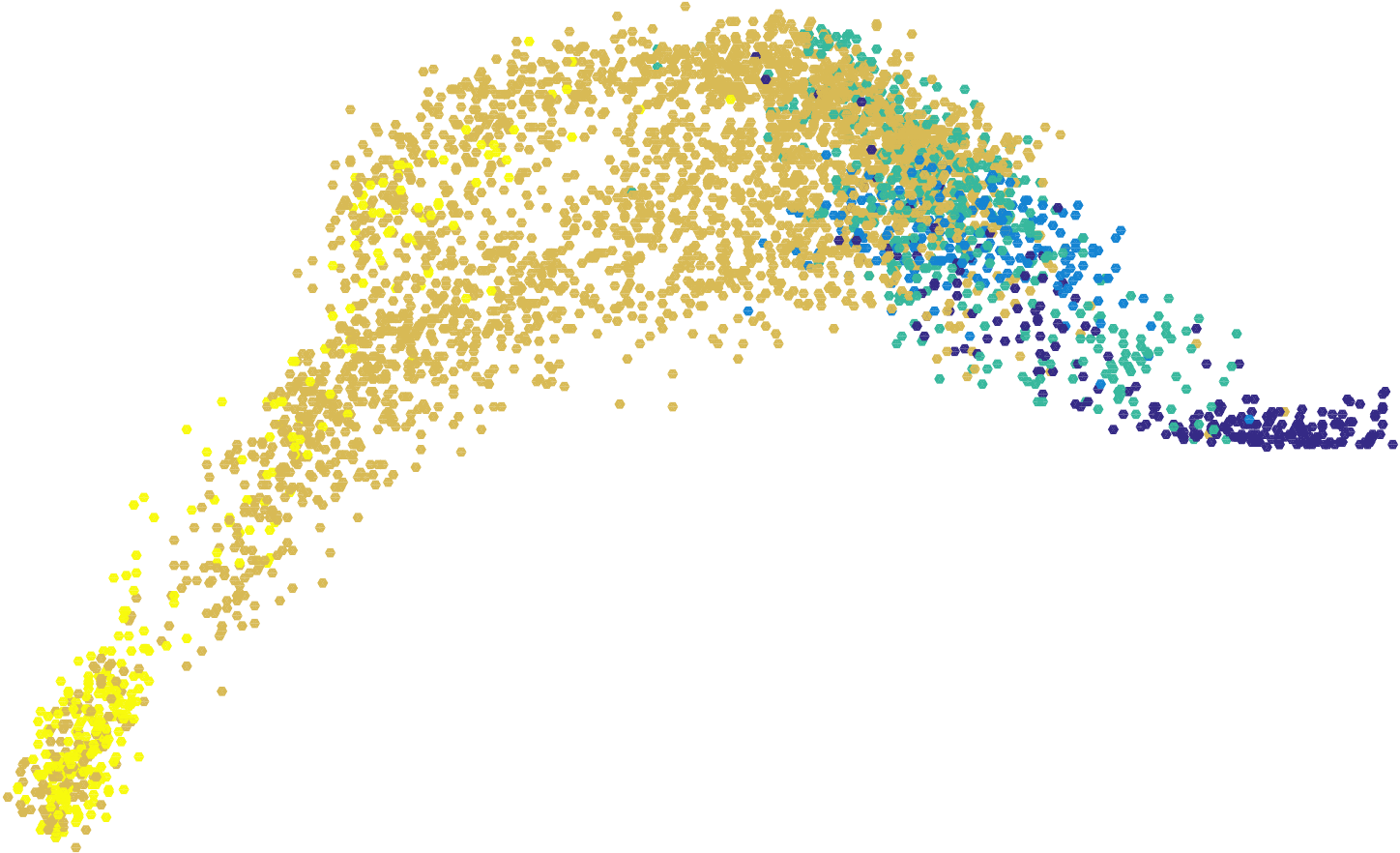}\label{fig:sleep_embed_CG_eeg}}
    \hspace{1cm}
    \subfloat[]
    {\includegraphics[scale=0.25]{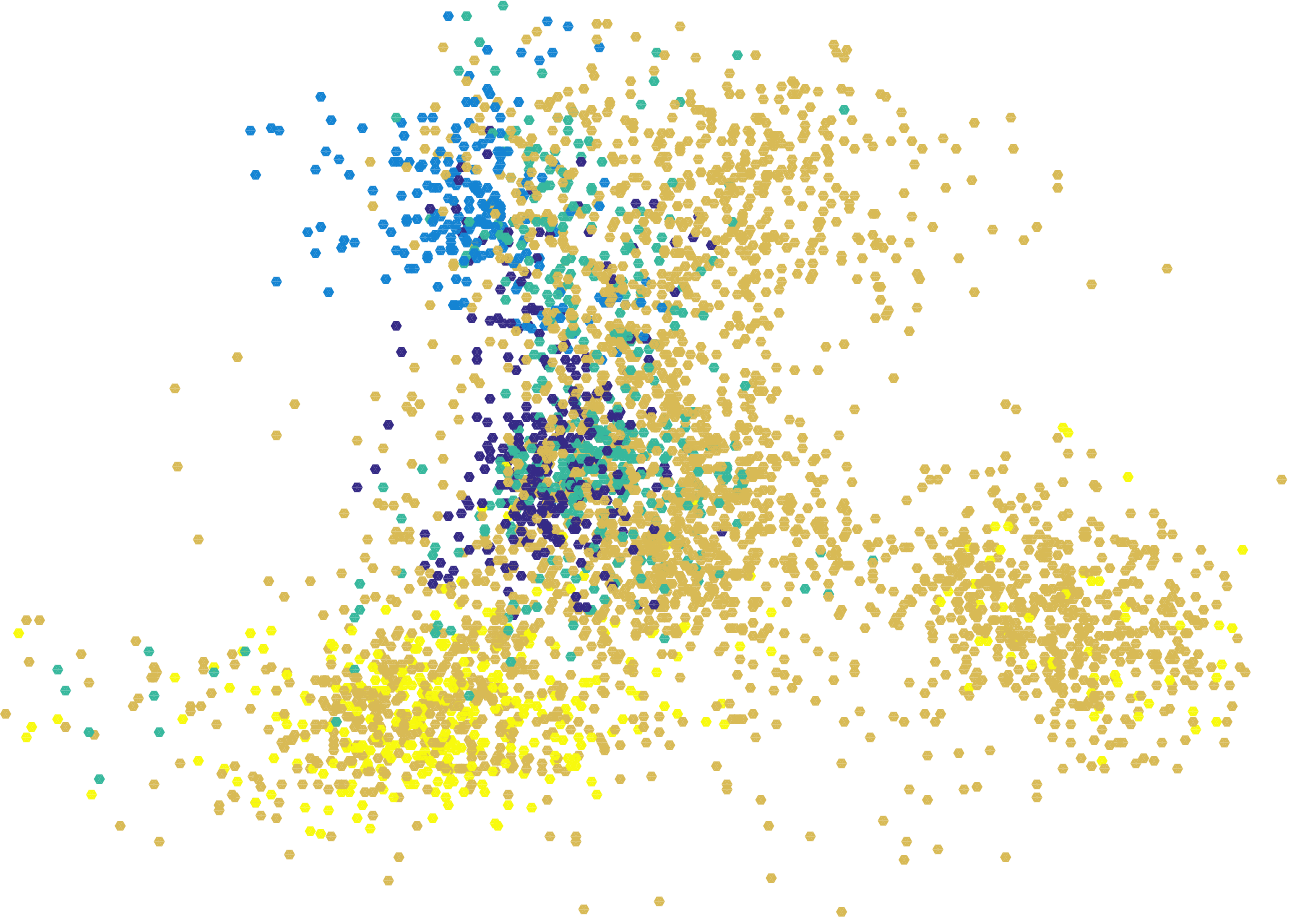}\label{fig:sleep_embed_CG_resp}}\\
   \caption{The 3D \glspl{RP} of the embeddings obtained by single-sensor \gls{DM} (top row), the concatenation scheme (second row), the multiplication scheme (third row), and Algorithm \ref{alg:CG} (bottom row). The points are colored according to the sleep stage. The embeddings are based on the O2A1 channel in (a) and on the airflow measurements in (b). From the second row to the bottom row, the embeddings on the left column are based on the EEG set, and on the right column are based on the respiratory set.}
   \label{fig:sleep_embed}
\end{figure}

Figure \ref{fig:sleep_embed} provides a visual illustration of the obtained parametrization with respect to the sleep stage.
By comparing the Figure \ref{fig:sleep_embed_DM_eeg} and Figure \ref{fig:sleep_embed_DM_resp} with Figure \ref{fig:sleep_embed_CG_eeg} and Figure \ref{fig:sleep_embed_CG_resp} we can observe the improvement achieved by the additional information obtained from combining information from multiple sensors.
In addition, by comparing Figures \ref{fig:sleep_embed_concat_eeg}-\ref{fig:sleep_embed_mul_resp} with Figure \ref{fig:sleep_embed_CG_eeg} and Figure \ref{fig:sleep_embed_CG_resp} we observe the improvement achieved by filtering out of the sensor-specific nuisance variables.
In these comparisons, it can be seen that the embeddings obtained by using the proposed algorithm results in a better parametrization of the sleep stage evaluation; different sleep states appear to be more separated, especially in the case of the respiratory signals.

To objectively assess the quality of the obtained embeddings, we use multi-class support vector machine (SVM). To ensure convergence and to prevent overfitting, we process only the $15$ most dominant eigenvectors from each embedding. It should be noted that due to the obtained fast decay of the eigenvalues, taking only the $15$ most dominant eigenvectors preserves the geometrical structure of the data.
We randomly partition the data into $2$ sets -- a training set (consisting of $75\%$ of the samples) and a validation set (consisting of $25\%$ of the samples).
The validation set contains $1,250$ time segments, which consist (on average) of $13.2\% (165)$ segments labeled as awake stage, $10\% (125)$ segments labeled as REM stage, $11.2\% (140)$ segments labeled as N1 stage, $49.6\% (620)$ segments labeled as N2 stage and $16\% (200)$ segments labeled as N3 stage.
The trained classifier is used to classify the sleep stage in the validation set.
We repeat this classification $10$ times, for different randomly chosen partitions of training and validation sets.
The average classification results for each scheme are depicted in Table \ref{tab:SVM_prediction_errors}.
The obtained classification results achieved by the proposed algorithm are superior compared to the obtained results from other schemes, both in the case of the EEG set and in the case of the respiratory set.
In these results, the advantages of proper filtering are evident, as it can be seen that in contrast to the proposed algorithm, the concatenation scheme and the multiplication scheme attain inferior classification results, and in some cases, their results are comparable to the results achieved by processing data from only a single sensor. In the case of the multiplication scheme this may be caused by too excessively filtering. In the case of the concatenation scheme, where no filtering is applied, this may be caused by the existence of interferences and noise.

The results in Table \ref{tab:SVM_prediction_errors} may provide additional insights related to the sleep dynamics that extend the scope of the evaluation of the algorithms.
The classification results achieved by the multiplication scheme are inferior comparing to the results achieved by single-sensor schemes in the case of the respiratory set, where as in the case of the EEG set the achieved results are similar to the single-sensor schemes. This supports the hypothesis that different EEG recording exhibit more homogeneous geometrical structures, with possibly less noise and fewer distortions, compared to the data acquired through the different respiratory recordings.

The homogeneity of the EEG set might explain another interesting observation stemming from these classification results.
In the case of the EEG set, we can see that combining the information acquired from multiple sensors using the proposed algorithm results with superior results, even compared to the results that would have been achieved using the best single-sensor scheme. This implies that the proposed algorithm manages to simultaneously cancel the effect of the additional noise-sensor as well as to properly integrate the information embodied in the multiple sensors.
Conversely, in the case of the respiratory set, we can see that even though the proposed algorithm manages to improve the results achieved by the CFlow channel, it did not manage to improve the results achieved by the single sensor schemes based on the ABD or the THO channel. Yet, it did manage to cancel the effect of the noise-sensor, but not as successfully as in the case of the EEG set. In this regard, it is worth emphasizing that the evaluation of the results from each sensor and from each scheme are based on unknown sleep stage labelling. Thus, the ``quality'' of the different sensors are not known in advance, and obtaining a result from our sensor fusion scheme that is comparable to the results attained by the best single sensor is still of value.
\begin{table} [t]
\centering
\caption{Classification results using SVM. The prediction errors (standard deviations) based on the different embeddings are presented. The total error (standard deviation) is calculated by a weighted mean of the prediction errors (standard deviations) in each sleep stage. (a) The classification based on the respiratory set. (b)  The classification based on the EEG set.}
\begin{tabular}{ |c||c|c|c|c|c||c| }
\hline
\multicolumn{7}{|c|}{Prediction errors based on the respiratory set} \\
\hline
Sensor/Scheme & Awake & REM & N1 & N2 & N3 & Total \\
\hline \hline
CFlow&0.383&0.258&0.545&0.179&0.244&0.264 (0.073)\\
ABD&0.299&0.154&0.426&0.162&0.185&0.209 (0.051)\\
THO&0.262&0.15&0.412&0.153&0.164&0.196 (0.051)\\
Noise&0.559&0.513&0.583&0.582&0.529&0.562 (0.039)\\
Concatenation scheme&0.476&0.474&0.552&0.516&0.5&0.507 (0.039)\\
Multiplication scheme&0.343&0.262&0.555&0.265&0.265&0.307 (0.069)\\
Common Graph&0.252&0.183&0.443&0.178&0.201&0.22 (0.048)\\
\hline
\multicolumn{7}{c}{(a)} \\
\multicolumn{7}{c}{} \\
\hline
\multicolumn{7}{|c|}{Prediction errors based on the EEG set} \\
\hline
Sensor/Scheme & Awake & REM & N1 & N2 & N3 & Total \\
\hline \hline
O1A2&0.267&0.281&0.624&0.132&0.273&0.25 (0.056)\\
O2A1&0.283&0.25&0.603&0.142&0.24&0.244 (0.069)\\
C4A1&0.305&0.273&0.623&0.139&0.291&0.258 (0.068)\\
C3A2&0.298&0.276&0.619&0.132&0.289&0.254 (0.06)\\
Noise&0.551&0.499&0.579&0.58&0.53&0.557 (0.042)\\
Concatenation scheme&0.342&0.325&0.499&0.307&0.325&0.34 (0.039)\\
Multiplication scheme&0.219&0.191&0.47&0.238&0.2&0.252 (0.045)\\
Common Graph&0.227&0.153&0.425&0.133&0.179&0.188 (0.036)\\

\hline
\multicolumn{7}{c}{(b)} \\
\multicolumn{7}{c}{} \\
\end{tabular}
\label{tab:SVM_prediction_errors}
\end{table}

Figure \ref{fig:sleep_embed_freq} further illustrates the poor embeddings and classification results achieved by the concatenation scheme. The same embeddings, which are depicted in Figure \ref{fig:sleep_embed}, are presented here, but this time with a different color -- now according to the instantaneous frequency of the noise sensor \eqref{eq:noise_sensor}.
As can be observed, in contrast to the embeddings achieved by the proposed algorithm or by the multiplication scheme, the embeddings achieved by the the concatenation scheme are well correlated with the instantaneous frequency in the noise sensor, indicating that the undelying structure is wrongly captured.
This further illustrates the difference between the filtering effects of our algorithm and other methods, which are based to the fusion of data from all the sensors \cite{davenport2010joint,keller_audio_visual_2010,salhov2016multi}.

\begin{figure}[t] \centering
    \subfloat[]
    {\includegraphics[scale=0.25]{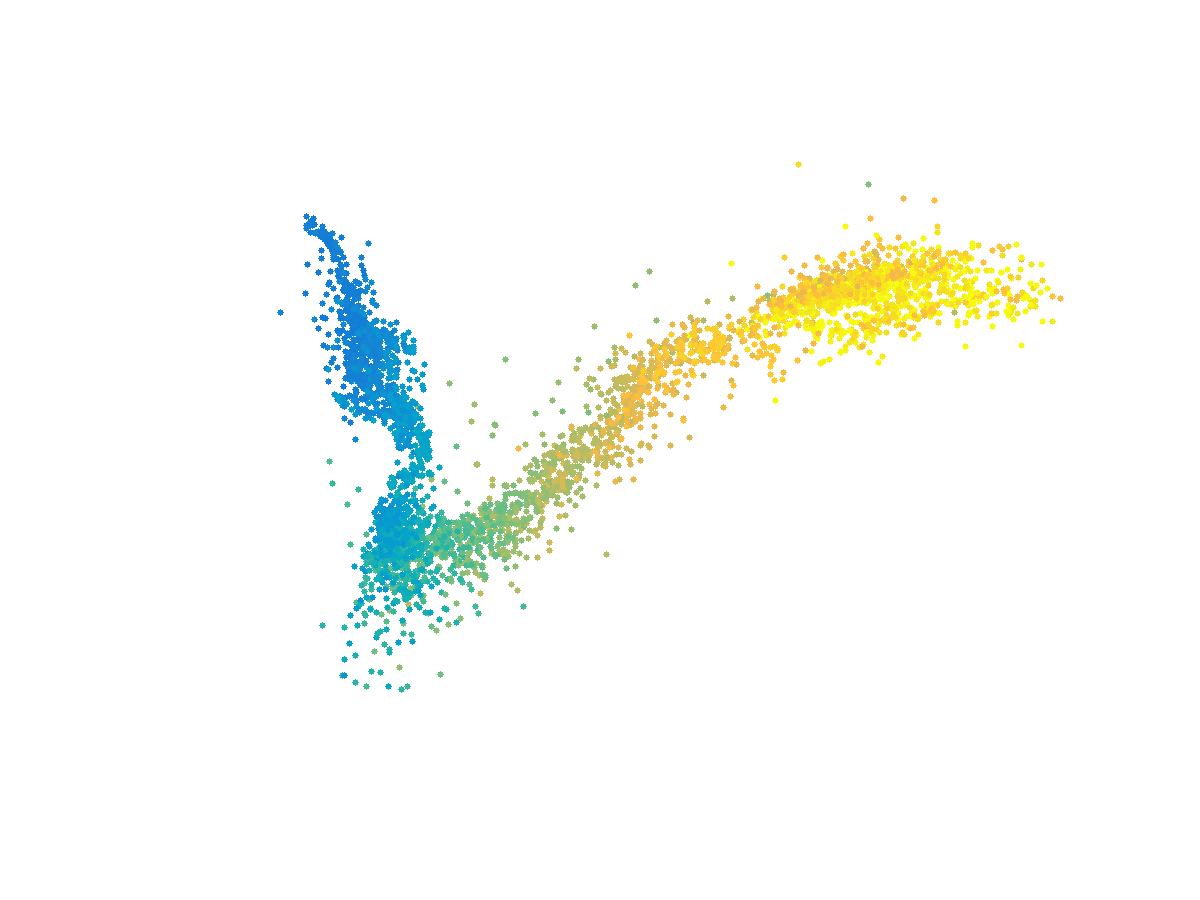}
    \label{fig:sleep_embed_concat_freq_eeg}}
    \hspace{1cm}
    \subfloat[]
    {\includegraphics[scale=0.25]{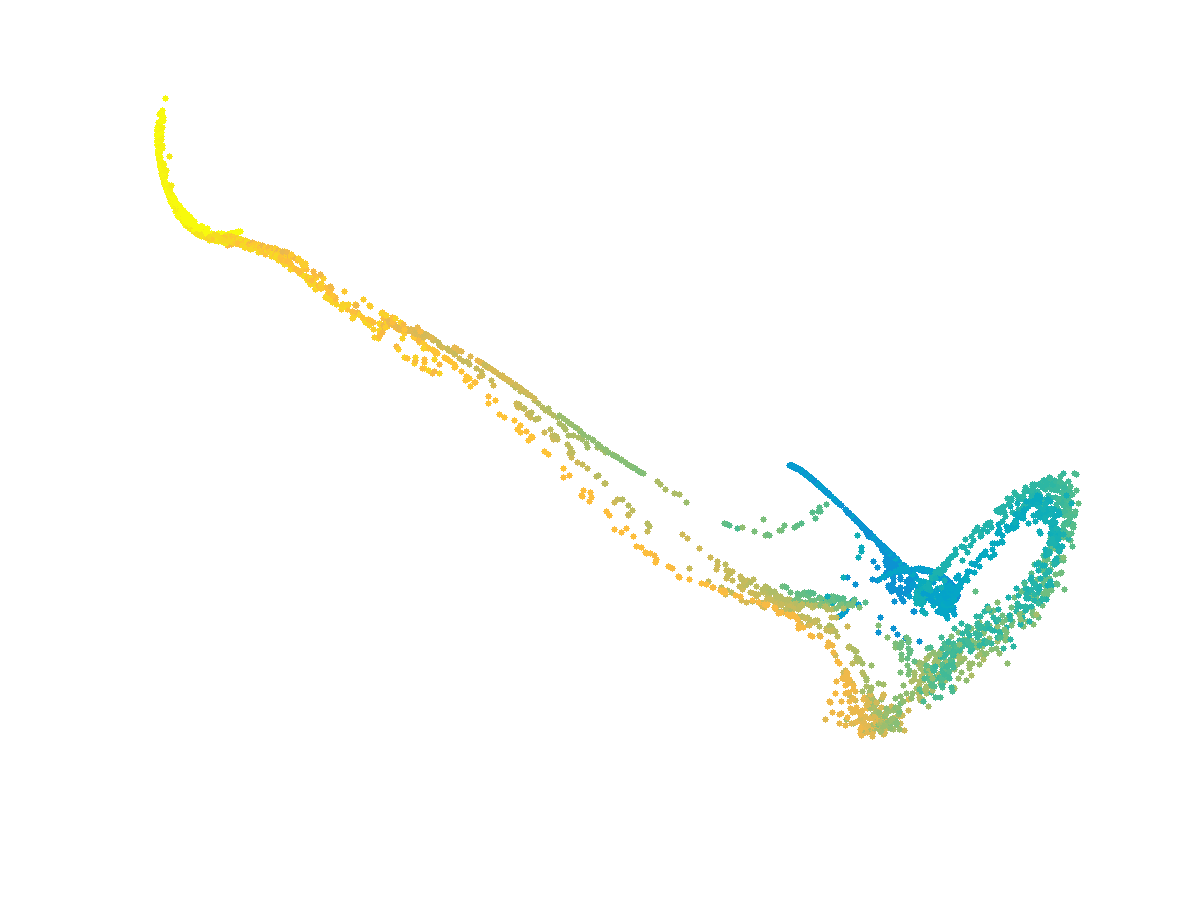}
    \label{fig:sleep_embed_concat_freq_resp}}
    \\
    \subfloat[]
    {\includegraphics[scale=0.25]{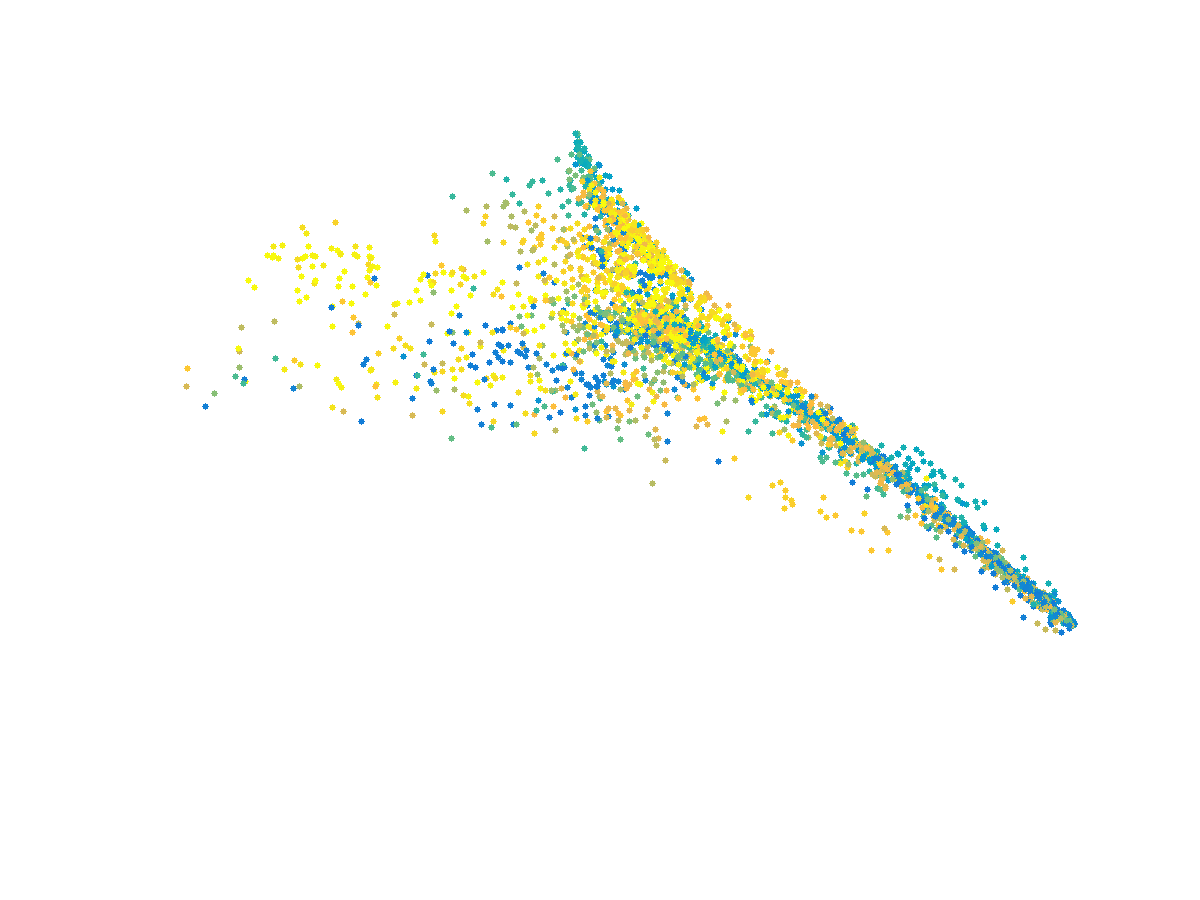}
    \label{fig:sleep_embed_mul_freq_eeg}}
    \hspace{1cm}
    \subfloat[]
    {\includegraphics[scale=0.25]{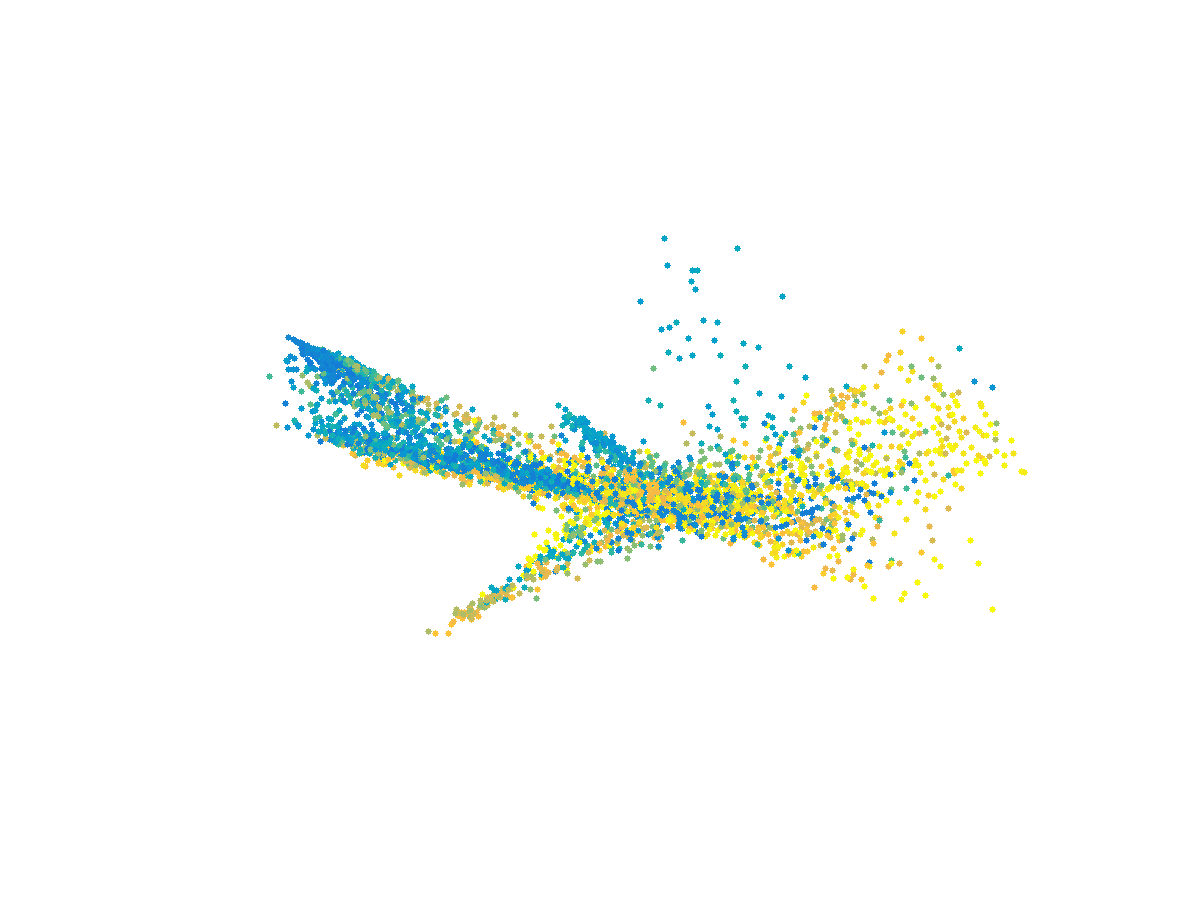}
    \label{fig:sleep_embed_mul_freq_resp}}
    \\
    \subfloat[]
    {\includegraphics[scale=0.25]{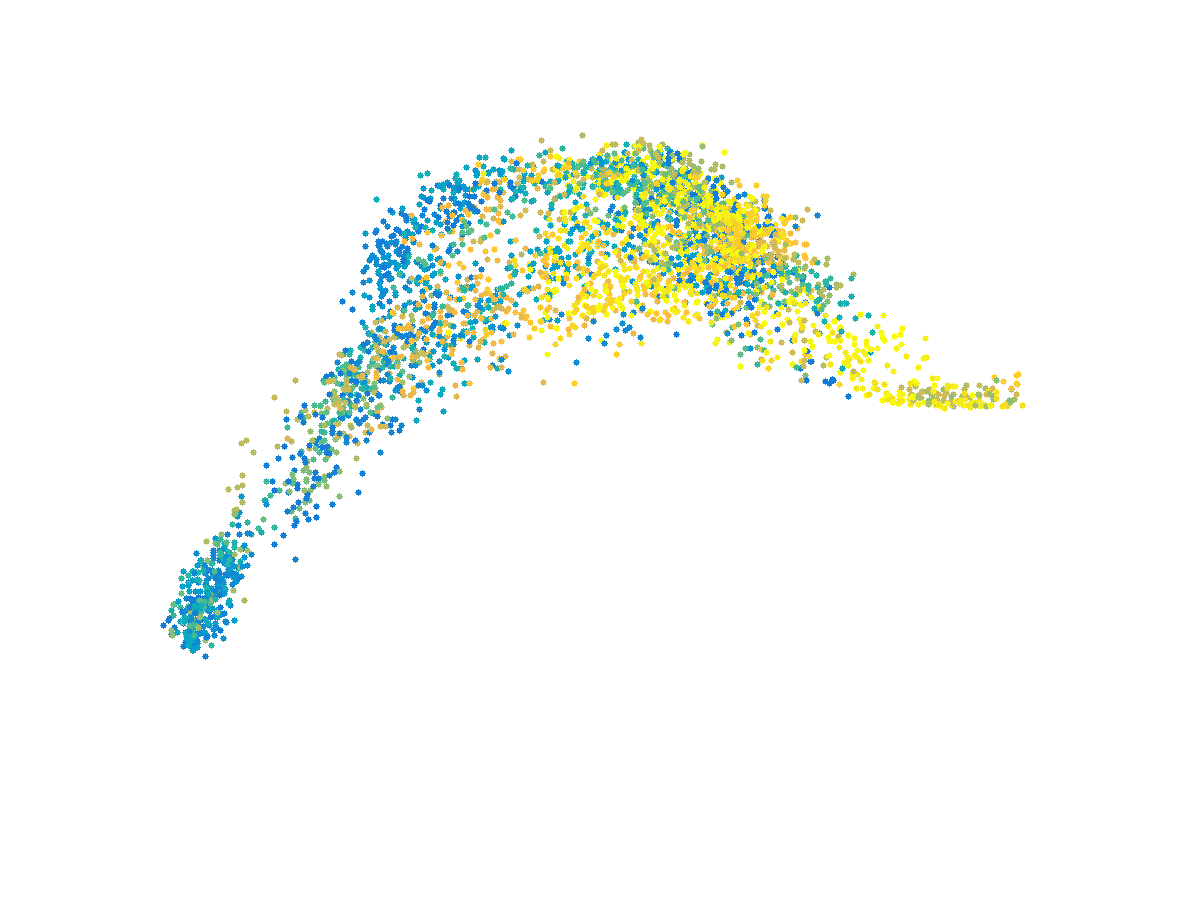}
    \label{fig:sleep_embed_CG_freq_eeg}}
    \hspace{1cm}
    \subfloat[]
    {\includegraphics[scale=0.25]{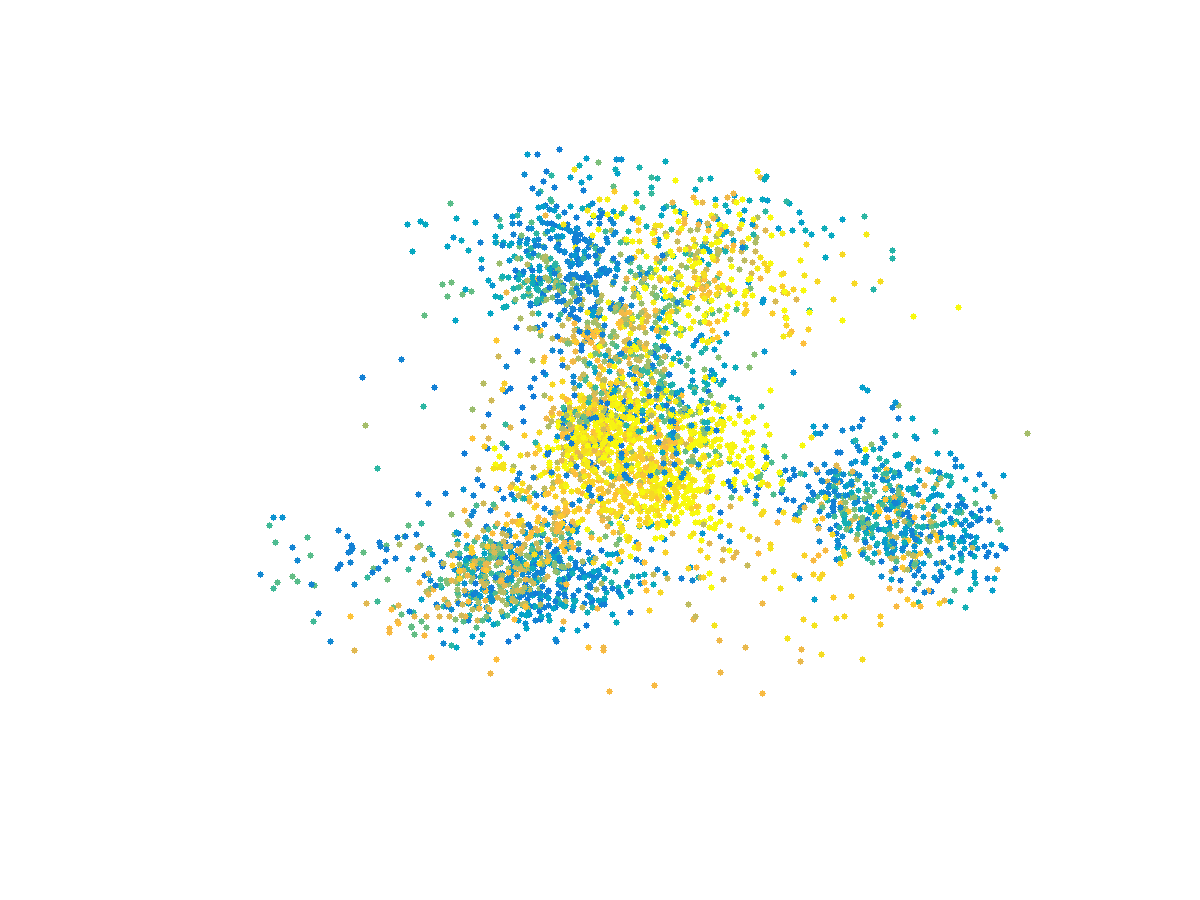}
    \label{fig:sleep_embed_CG_freq_resp}}
    \\
    \caption{The same embeddings as in Figure \ref{fig:sleep_embed}, colored according to the instantaneous frequency of the noise sensor.}
    \label{fig:sleep_embed_freq}
\end{figure}

\section{Conclusions}
\label{sec:Conclusions}
In this paper, we propose a new algorithm for fusing information measured by multiple, multimodal sensors. The primary focus is on a setting in which all sensors observe the same system, but each introduces different variables -- some are related to various aspects of the system of interest, whereas others are sensor-specific and irrelevant.
We present a nonlinear data fusion scheme for suppressing the sensor-specific variables while preserving the system variables measured by two or more sensors. Experimental results demonstrate the applicability of our method to artificial toy problem and to recorded multimodal data for the purpose of sleep stage assessment.

The core of the presented technique is an implementation of an abstract notion of intersection and union of multimodal data sets. While the intersection between two sets is well defined and theoretically explained \cite{talmon2016latent}, the union of two (or more) sets still calls for rigorous analysis. Future work will include such analysis and the development of a union scheme that respects uniqueness.

\bibliographystyle{plain}
\bibliography{papers}

\end{document}